\documentclass{article}
\usepackage{amsmath}
\usepackage{amsthm}
\usepackage{hyperref}
 \usepackage{amssymb}
 \usepackage{xcolor}
 \usepackage{graphicx}
 \usepackage{mathtools}
 \usepackage{graphicx}
 \usepackage{float}
 \usepackage{tikz}
\usetikzlibrary{calc}
\usetikzlibrary{arrows.meta}
\usepackage{subcaption}
\usepackage{pgfplots}
\usepackage{booktabs}
\RequirePackage{natbib}

\usepackage{array}
\usepackage{booktabs}
\usepackage{pifont}
\usepackage{centernot}

\usepackage{algorithm}
\usepackage{algpseudocode}
\usepackage{caption}

\definecolor{fieldPurple}{RGB}{135,31,255}
\definecolor{sampleColor}{RGB}{0,150,140}
\definecolor{fieldGreen}{RGB}{0, 180, 120}
\definecolor{BackgroundGraph}{RGB}{234,234,242}
\definecolor{BlueCustom}{RGB}{66,133,244}
\definecolor{RedCustom}{RGB}{219,68,55}
\definecolor{OrangeCustom}{RGB}{188,88,42}
\definecolor{Blue_field2}{RGB}{135,31,255}
\definecolor{Green_field2}{RGB}{0, 180, 120}

\usepackage{amsmath,amssymb}
\usepackage{tikz}
\usepackage{subcaption}
\usepackage[margin=1in]{geometry}
\usetikzlibrary{arrows.meta,calc}

\definecolor{BlueField}{RGB}{46,98,180}
\definecolor{GreenField}{RGB}{50,145,95}
\definecolor{GraySphere}{RGB}{180,180,180}

\DeclareMathOperator{\sech}{sech}

\newcommand{\DrawHemisphere}{
    \begin{scope}
        \clip (-1.05,-0.02) rectangle (1.05,1.08);
        \shade[ball color=gray!18,opacity=0.32] (0,0) circle (1);
    \end{scope}
    \draw[gray!75, line width=0.8pt] (-1,0) arc[start angle=180,end angle=0,radius=1];
}

\newcommand{\DrawBlueFamily}[1]{
    \foreach \yy/\rx/\ry in {
        0.10/0.95/0.10,
        0.26/0.88/0.12,
        0.43/0.76/0.12,
        0.60/0.58/0.10,
        0.77/0.34/0.07
    }{
        \draw[BlueField, line width=0.9pt, opacity=#1]
            (-\rx,\yy)
            arc[start angle=180,end angle=0,x radius=\rx,y radius=\ry];
    }
}

\newcommand{\DrawGreenFamily}[1]{

    \foreach \xc/\rx/\ry in {
        -0.78/0.18/0.62,
        -0.42/0.13/0.82
    }{
        \draw[
            GreenField,
            line width=0.9pt,
            opacity=#1
        ]
        plot[
            domain=0:90,
            samples=50
        ]
        (
            {\xc - \rx*sin(\x)},
            {\ry*cos(\x)}
        );
    }

    \draw[
        GreenField,
        line width=0.9pt,
        opacity=#1
    ]
    (0,0) -- (0,1.0);

    \foreach \xc/\rx/\ry in {
        0.42/0.13/0.82,
        0.78/0.18/0.62
    }{
        \draw[
            GreenField,
            line width=0.9pt,
            opacity=#1
        ]
        plot[
            domain=0:90,
            samples=50
        ]
        (
            {\xc + \rx*sin(\x)},
            {\ry*cos(\x)}
        );
    }
}

\newtheorem{theorem}{Theorem}[section]
\newtheorem{proposition}{Proposition}[section]

\newtheorem{definition}{Definition}[section]
\newtheorem{remark}{Remark}[section]
\newtheorem{lemma}{Lemma}[section]
\usetikzlibrary{decorations.markings}

\tikzset{
  midarrow/.style={
    postaction={
      decorate,
      decoration={
        markings,
        mark=at position 0.55 with{\arrow[scale=1.0]{Stealth}}
      }
    }
  }
}

\tikzset{
  mycross/.pic={
    \draw[pic actions] 
      (-0.5pt,0) -- (0.5pt,0)
      (0,-0.5pt) -- (0,0.5pt);
  },
}

\usepackage[preprint]{tmlr}

\pgfplotsset{compat=1.18}

\title{Discovering Data Manifold Geometry through Geometric Properties}

\author{
    \name David Vigouroux \\
    \addr IRT Saint Exupery, ANITI, Toulouse, France \\
    \addr IMT Atlantique, LaTIM UMR 1101 INSERM, Brest, France
    \AND
    \name Lucas Drumetz \\
    \addr IMT Atlantique, Lab-STICC, UMR CNRS 6285, Plouzané, France \\
    \addr INRIA, ODYSSEY team-project, Brest, France \\
    \addr SequoIA
    \AND
    \name Ronan Fablet \\
    \addr IMT Atlantique, Lab-STICC, UMR CNRS 6285, Plouzané, France \\
    \addr INRIA, ODYSSEY team-project, Brest, France \\
    \addr SequoIA
    \AND
    \name François Rousseau \\
    \addr IMT Atlantique, LaTIM UMR 1101 INSERM, Brest, France \\
    \addr SequoIA
}

\def\month{MM}  
\def\year{YYYY} 
\def\openreview{\url{https://openreview.net/forum?id=XXXX}} 

\begin{document}

\maketitle

\begin{abstract}
A prevailing paradigm in modern representation learning is the map-first approach, in which a representation map is learned from reconstruction, embedding, or task objectives. At the optimum, when the learned map accurately recovers a global coordinate chart, it should exhibit three structural properties whose geometric meaning can be illustrated through a face-editing example: \textit{Commutativity} requires that changing pose and then expression gives the same result as applying them in the reverse order; \textit{Time Coherence} requires that the same variation along one coordinate induces the same expression change across faces; \textit{Common-Reference} requires that all faces are organized relative to a common reference face. However, small approximation errors in the learned map need not translate into small errors in these structural properties, and can therefore disrupt the global organization of the representation. Based on this observation, we consider the converse of the map-first formulation and ask whether a global representation can instead emerge by directly learning these properties. We represent variations along individual coordinates through vector fields defined in the ambient space and introduce a non-contraction condition preventing one transformation from destroying directions associated with the others. We derive an unsupervised objective that learns these structural properties and establish theoretical results connecting its minimization to tangent-space recovery. Experiments on controlled manifolds validate the predicted tangent-space recovery and global structure, while an autoencoder baseline shows that small map-first errors can still produce substantial violations of the targeted properties.
\end{abstract}
 
\section{Introduction}
Under the manifold hypothesis, high-dimensional data are assumed to concentrate near a low-dimensional manifold. When this manifold is known \emph{a priori} or provided with a predefined parametrization, its geometric structure is fixed and learning reduces to approximating functions or embeddings defined on it. In contrast, when the manifold itself is unknown and only accessible through sampled data, any representation intended to reflect its intrinsic organization must infer this geometry from the data~\citep{manifoldlearningwhathow}. In particular, the tangent directions of the manifold describe the locally admissible variations of the data and therefore provide a natural geometric object from which to construct a representation.

Learning such directions is particularly valuable when variations in representation space are expected to carry semantic meaning. Directions or coordinates are widely used to represent concepts in interpretable models~\citep{kim2018tcav,fel2023craft,pmlr-v119-koh20a} and to define controllable transformations in generative models~\citep{Shen_2020_CVPR}. This suggests using the tangent geometry of the data manifold as a principled basis for constructing meaningful feature directions. 

However, tangent directions are inherently local objects: they describe how the data can vary around a given sample, but do not specify how these local variations should be connected across the manifold. Turning them into a global representation therefore requires additional consistency. Following different feature directions should lead to compatible transformations, the amount of variation along each direction should retain the same meaning across samples, and these transformations should organize the entire manifold within a common coordinate system. This perspective leads us to a unified notion of global consistency based on three complementary defined hereafter properties---\emph{commutativity}, \emph{time coherence}, and a \emph{common reference}. Related notions have been studied separately and linked to consistent, compositional feature manipulation---for instance, in face editing--- and intervention~\citep{Aoshima_2023_CVPR,higgins2018towards,pmlr-v202-kaba23a}. They can be understood jointly as manifestations of the same underlying global coordinate structure. 

\begin{figure}[t]
\centering

\begin{subfigure}[t]{0.32\linewidth}
\centering
\begin{tikzpicture}[scale=2.0]

    \coordinate (P) at (0.0,1.0);
    \DrawHemisphere
    \DrawBlueFamily{0.18}
    \DrawGreenFamily{0.18}

    \coordinate (A) at (-0.52,0.27);
    \coordinate (B) at ( 0.28,0.29);
    \coordinate (C) at (-0.35,0.63);
    \coordinate (D) at ( 0.16,0.65);

    \node[font=\scriptsize] at ($(P)+(-0.8,0.0)$) {(a)};
    
    \draw[BlueField, ultra thick, -{Stealth[length=2.5mm]}]
        (A) to[out=6,in=176] (B);

    \draw[GreenField, ultra thick, -{Stealth[length=2.5mm]}]
        (B) to[out=82,in=300] (D);

    \draw[GreenField, ultra thick, dashed, -{Stealth[length=2.5mm]}]
        (A) to[out=74,in=242] (C);

    \draw[BlueField, ultra thick, dashed, -{Stealth[length=2.5mm]}]
        (C) to[out=8,in=178] (D);

    \fill (A) circle (1.3pt);
    \fill (B) circle (1.1pt);
    \fill (C) circle (1.1pt);
    \fill (D) circle (1.3pt);

    \node[font=\scriptsize, below left] at (A) {$x$};
    \node[font=\scriptsize, below right] at (B) {$\Phi_1^{t_1}(x)$};
    \node[font=\scriptsize, above left] at (C) {$\Phi_2^{t_2}(x)$};
    \node[font=\scriptsize, above right] at (D) {$x'$};

    \node[BlueField, font=\scriptsize] at (-0.10,0.20) {$F_1$};
    \node[GreenField, font=\scriptsize] at (0.33,0.49) {$F_2$};
    \node[GreenField, font=\scriptsize] at (-0.46,0.48) {$F_2$};
    \node[BlueField, font=\scriptsize] at (-0.06,0.72) {$F_1$};

\end{tikzpicture}
\end{subfigure}
\hfill
\begin{subfigure}[t]{0.32\linewidth}
\centering
\begin{tikzpicture}[scale=2.0]

    \DrawHemisphere
    \DrawBlueFamily{0.18}
    \DrawGreenFamily{0.18}

    \draw[BlueField, ultra thick]
        (-0.82,0.28)
        to[out=8,in=180]
        (-0.05,0.41)
        to[out=0,in=172, -{Stealth[length=2.6mm]}]
        (0.78,0.30);

    \node[font=\scriptsize] at ($(P)+(-0.8,0.0)$) {(b)};
    \coordinate (P0) at (-0.68,0.30);
    \coordinate (P1) at ( 0.00,0.41);
    \coordinate (P2) at ( 0.62,0.32);

    \fill (P0) circle (1.3pt);
    \fill (P1) circle (1.3pt);
    \fill (P2) circle (1.3pt);

    \node[font=\scriptsize, below left] at ($(P0)+(0,0.2)$) {$x$};
    \node[font=\scriptsize, above] at ($(P1) + (0, -0.02)$) {$\Phi_1^{\tau}(x)$};
    \node[font=\scriptsize, below right] at ($(P2) + (0,0.2)$) {$\Phi_1^{2\tau}(x)$};

    \node[BlueField, font=\scriptsize, text opacity=1, inner sep=1pt]
        at (0.02,0.67) {$T_1(\Phi_1^t(x))=T_1(x)-t$};

    \node[GreenField, font=\scriptsize, text opacity=1, inner sep=1pt]
        at (0.00,0.05) {$T_2=\mathrm{const.}$};

    \node[font=\scriptsize] at (-0.68,0.20) {$T_1=a$};
    \node[font=\scriptsize] at ( 0.00,0.30) {$T_1=a-\tau$};
    \node[font=\scriptsize] at ( 0.72,0.22) {$T_1=a-2\tau$};

\end{tikzpicture}
\end{subfigure}
\hfill
\begin{subfigure}[t]{0.32\linewidth}
\centering
\begin{tikzpicture}[scale=2.0]
    \DrawHemisphere
    \DrawBlueFamily{0.18}
    \DrawGreenFamily{0.18}

    \fill (P) circle (1.3pt);

    \node[font=\scriptsize] at ($(P)+(-0.8,0.0)$) {(c)};
    
    \node[BlueField, font=\scriptsize, below right] at ($(P)+(0.15,-0.15)$) {$F_1$};
    \node[GreenField, font=\scriptsize, above left] at ($(P)+(0.00,-0.33)$) {$F_2$};

    \node[font=\scriptsize, above] at ($(P)+(0,0.03)$) {$c$};

    \coordinate (X) at (0.52,0.63);
    \coordinate (Y) at (0.00,0.63);
    \coordinate (Z) at (0.52,1.00);

    \fill (X) circle (1.1pt);
    \node[font=\scriptsize, above right] at (X) {$x$};

    \draw[BlueField, thick, dashed, -{Stealth[length=2.0mm]}]
        (X) to[out=180,in=0] (Y);
    \draw[GreenField, thick, dashed, -{Stealth[length=2.0mm]}]
        (Y) to[out=90,in=270] (P);

\end{tikzpicture}
\end{subfigure}

\caption{
Illustration of the consistency properties induced by a global coordinate map on the upper hemisphere. (a) \textbf{Commutativity.} Following \(F_1\) then \(F_2\), or \(F_2\) then \(F_1\), leads to the same endpoint, modeling an independence criterion between the two directions. (b) \textbf{Time coherence.} Along the flow of \(F_1\), only the coordinate \(T_1\) changes while \(T_2\) remains invariant, expressing an independence criterion between the associated coordinates. (c) \textbf{Common reference.}
A global map organizes the manifold around a shared reference point \(c\). From any sample, following the coordinate flows leads to the same common center. }
\label{fig:global_map_commutativity_time}
\end{figure}

Concretely, as illustrated in Figure~\ref{fig:global_map_commutativity_time}, each property captures a distinct aspect of this global consistency. First, \textit{Commutativity} requires different feature transformations to be independent of the order in which they are applied. For instance, when editing a face, adding a smile and then changing the head pose should lead to the same result as changing the pose first and adding the smile afterward. \textit{Time Coherence} requires each transformation to use the same scale across the dataset. For example, if one coordinate measures mouth opening, increasing this coordinate by the same amount for two different faces should produce a comparable increase in mouth opening for both, without simultaneously changing attributes such as head pose or eye opening. Finally, a \textit{Common Reference} requires all samples to be expressed relative to the same coordinate origin. In the face example, this means that the coordinates for attributes such as mouth opening, head pose, and eye opening are all measured from the same reference face. Together, these three properties turn locally defined variations into a globally organized and consistently interpretable representation. Here, we consider these properties jointly through the differential-geometric structure of coordinate flows.

These properties have a direct interpretation in differential geometry. Coordinate vector fields commute and their flows act as translations of the associated coordinates~\citep[Eq.~(8.10), p.~187; Theorem~9.44, p.~233]{lee_introduction_2012}; conversely, independent commuting vector fields locally admit coordinates in which they become coordinate directions~\citep[Theorem~9.46, p.~234]{lee_introduction_2012}. When these flows additionally admit a common accessible origin, they define the global flow-based coordinate structure considered in this work. Under the regularity assumptions formalized in Theorem~\ref{thm:global_coordinates_coherence}, this gives the characterization
\[
\boxed{
\text{Global flow-based coordinates}
\quad\Longleftrightarrow\quad
\text{Commutativity + Time coherence + Common reference}.
}
\]
Figure~\ref{fig:global_map_commutativity_time} illustrates these three properties and their relation to a global coordinate system.

In practice, however, the manifold is only observed through finitely many samples and the coordinate identities can only be enforced approximately. The exact equivalence above therefore does not immediately guarantee that small empirical violations yield a good approximation of the desired global coordinates. This limitation also affects standard autoencoders: minimizing the reconstruction error $D(E(x))\approx x$ ensures that samples can be reconstructed, but does not guarantee that the learned coordinates satisfy the structural properties required for global consistency. An arbitrarily small reconstruction error does not imply an equally accurate approximation of commutativity, time coherence, or common-reference consistency. We formalize these approximation limitations through a counterexample in Appendix~\ref{app:ae_counterexample} and observe the failure of reconstruction alone to recover the desired coordinate structure empirically in Table~\ref{table_toys}.

These observations raise a more general question about how geometric representations should be learned: \emph{can a global representation be identified and learned from the structural properties it is required to satisfy, rather than from a direct objective on the representation map itself?} Standard approaches are predominantly map-first: they optimize a representation according to reconstruction, embedding, or task-specific criteria, and desirable geometric properties are subsequently analyzed or encouraged through additional regularization. We investigate the reverse perspective. In the flow-based setting considered here, we ask whether global coordinate structure can instead emerge as a consequence of enforcing its defining properties.

\begin{figure*}
\centering
\begin{subfigure}{0.32\textwidth}
\centering
\begin{tikzpicture}[scale=1.8, >=stealth]
  \node[font=\small] at (-1,1.1) {(a)};

\begin{scope}
  \clip (-1.,0) rectangle (1.,1.0); 

  \shade[ball color=gray!15, opacity=0.30] (0,0) circle (1.);

\foreach \angle in {-160, -120, -90, -60,-40,-20,0,20,40,60}{
    \draw[gray!50, thin, opacity=0.5]
      plot[domain=0:180, samples=80]
        ({1.*sin(\x)*cos(\angle)}, {1.*cos(\x)});
  }

  \foreach \h in {-0.9,-0.8,-0.4,0,0.4}{
    \draw[gray!50, thin, opacity=0.5]
      ({sqrt(1 - (\h/1.)^2)*1.}, {\h})
        arc[start angle=0, end angle=180,
            x radius={sqrt(1 - (\h/1.)^2)*1.},
            y radius={0.5}];
  }
\end{scope}

  \coordinate (C) at (-0.2,1.2);      
  \coordinate (base) at (0,0.05);  

  \foreach \a/\lab in {-1.0/, -0.8/, -0.6/, -0.4/, -0.2/, -0.0/, 0.2/, 0.4/, 0.6/, 0.8/, 1.0/}{
    \coordinate (p) at (\a,0.0);
    \draw[fieldPurple, very thick, midarrow]
      (p) .. controls (\a*0.7,0.55) ..
      (C);
  }
  
 \node[font=\small] at (-0.05,1.2) {$C$};
 
\node[fieldPurple, font=\scriptsize] at (0.55,1.)
    {$\phi_F^{T}(x)$};

  \draw[red,
      arrows={Stealth[length=1.5pt, width=1.5pt]-Stealth[length=1.5pt, width=1.5pt]},
      line width=0.3pt
    ]
    (-0.06,0.75) -- (0.02,0.76);
  \draw[red,
      arrows={Stealth[length=2pt, width=2pt]-Stealth[length=2pt, width=2pt]},
      line width=0.3pt
    ]
    (-0.01,0.55) -- (0.099,0.57);
  \draw[red,
      arrows={Stealth[length=2pt, width=2pt]-Stealth[length=2pt, width=2pt]},
      line width=0.3pt
    ]
    (0.0,0.4) -- (0.13,0.42);
  \draw[red,
      arrows={Stealth[length=2pt, width=2pt]-Stealth[length=2pt, width=2pt]},
      line width=0.3pt
    ]
    (0.01,0.25) -- (0.16,0.27);
  \draw[red,
      arrows={Stealth[length=2pt, width=2pt]-Stealth[length=2pt, width=2pt]},
      line width=0.3pt
    ]
    (0.01,0.12) -- (0.17,0.14);

    \begin{axis}[
    width=2.5cm,
    height=2.5cm,
    axis equal,
    xmin=0, xmax=1,
    ymin=0, ymax=1,
    axis lines=none,   
    ticks=none,        
    color=sampleColor
]

    \addplot[
        only marks,
        mark=*,
        mark size=0.1pt
    ] 
    table [
        col sep=space        
    ] {random_points.csv};
    
\end{axis}

\end{tikzpicture}
\phantomsubcaption
\label{fig::main_method_a}
\end{subfigure}
\begin{subfigure}{0.32\textwidth}
\centering
\begin{tikzpicture}[scale=1.8, >=stealth]
  \node[font=\small] at (-1,1.1) {(b)};

  \begin{scope}
  \clip (-1.,0) rectangle (1.,1.0);

  \shade[ball color=gray!15, opacity=0.30] (0,0) circle (1.);

\foreach \angle in {-160, -120, -90, -60,-40,-20,0,20,40,60}{
    \draw[gray!50, thin, opacity=0.5]
      plot[domain=0:180, samples=80]
        ({1.*sin(\x)*cos(\angle)}, {1.*cos(\x)});
  }

  \foreach \h in {-0.9,-0.8,-0.4,0,0.4}{
    \draw[gray!50, thin, opacity=0.5]
      ({sqrt(1 - (\h/1.)^2)*1.}, {\h})
        arc[start angle=0, end angle=180,
            x radius={sqrt(1 - (\h/1.)^2)*1.},
            y radius={0.5}];
  }

    \foreach \h in {-0.2, 0.0, 0.2, 0.4}{
    \draw[fieldPurple, very thick, midarrow]
      ({sqrt(1 - (\h/1.)^2)*1.}, {\h})
        arc[start angle=0, end angle=90,
            x radius={sqrt(1 - (\h/1.)^2)*1.},
            y radius={0.5}];
            
    \pic[line width=3pt,red] at (0,\h+0.5) {mycross};
  }
  
\end{scope}

  \def\lat{0.55} 
  \pgfmathsetmacro{\radlat}{sqrt(1 - (\lat/1.05)^2)*1.05}

  \draw[red, <->, line width=0.3pt]
    (0.1,0.5) -- (0.1,0.70);

  \draw[red, <->, line width=0.3pt]
    (0.4,0.46) -- (0.4,0.65);

    \draw[red, <->, line width=0.3pt]
    (0.65,0.4) -- (0.65,0.57);

  \draw[red, <->, line width=0.3pt]
    (0.8,0.3) -- (0.8,0.48);

    \begin{axis}[
    width=2.5cm,
    height=2.5cm,
    axis equal,
    xmin=0, xmax=1,
    ymin=0, ymax=1,
    axis lines=none,   
    ticks=none,        
    color=sampleColor
]

    \addplot[
        only marks,
        mark=*,
        mark size=0.1pt
    ] 
    table [
        col sep=space        
    ] {random_points.csv};
\end{axis}  

\node[fieldPurple, font=\scriptsize] at (0.,1.1)
    {$F_1 : \mathcal{L}_{F_1} g_1 \succeq 0$};

\end{tikzpicture}
\phantomsubcaption
\label{fig::main_method_b}
\end{subfigure}
\begin{subfigure}{0.32\textwidth}
\centering
\begin{tikzpicture}[scale=1.8, >=stealth]
  \node[font=\small] at (-1,1.1) {(c)};

  \begin{scope}
  \clip (-1.,0) rectangle (1.,1.0);

  \shade[ball color=gray!15, opacity=0.30] (0,0) circle (1.);

\foreach \angle in {-160, -120, -90, -60,-40,-20,0,20,40,60}{
    \draw[gray!50, thin, opacity=0.5]
      plot[domain=0:180, samples=80]
        ({1.*sin(\x)*cos(\angle)}, {1.*cos(\x)});
  }

  \foreach \h in {-0.9,-0.8,-0.4,0,0.4}{
    \draw[gray!50, thin, opacity=0.5]
      ({sqrt(1 - (\h/1.)^2)*1.}, {\h})
        arc[start angle=0, end angle=180,
            x radius={sqrt(1 - (\h/1.)^2)*1.},
            y radius={0.5}];
  }

    \foreach \h in {-0.2, 0.0, 0.2, 0.4}{
    \draw[fieldPurple, very thick, midarrow]
      ({sqrt(1 - (\h/1.)^2)*1.}, {\h})
        arc[start angle=0, end angle=90,
            x radius={sqrt(1 - (\h/1.)^2)*1.},
            y radius={0.5}];
            
    \pic[line width=3pt,red] at (0,\h+0.5) {mycross};
  }
  
\end{scope}

  \def\lat{0.55} 
  \pgfmathsetmacro{\radlat}{sqrt(1 - (\lat/1.05)^2)*1.05}

  \draw[red, <->, line width=0.3pt]
    (0.1,0.5) -- (0.1,0.70);

  \draw[red, <->, line width=0.3pt]
    (0.4,0.46) -- (0.4,0.65);

    \draw[red, <->, line width=0.3pt]
    (0.65,0.4) -- (0.65,0.57);

  \draw[red, <->, line width=0.3pt]
    (0.8,0.3) -- (0.8,0.48);

  \draw[fieldGreen, line width=1.pt, midarrow]
        (0,0) -- (0,1.0);

\begin{axis}[
    width=2.5cm,
    height=2.5cm,
    axis equal,
    xmin=0, xmax=1,
    ymin=0, ymax=1,
    axis lines=none,   
    ticks=none,        
    color=sampleColor
]

    \addplot[
        only marks,
        mark=*,
        mark size=0.1pt
    ] 
    table [
        col sep=space        
    ] {random_points.csv};
\end{axis}    

\node[fieldPurple, font=\scriptsize] at (0.75,0.9)
    {$F_1$};
\node[fieldGreen, font=\scriptsize, rotate=90] at (-0.15,0.5)
    {$F_2$};

\node[fieldPurple, font=\scriptsize] at (0.,1.15) {Combine fields};

\node[font=\small, red] at (-0.15,0.95) {$C$};
\pic[line width=4pt,red] at (0,1.0) {mycross};

\end{tikzpicture}
\end{subfigure}

\caption{\textbf{Main Principle of Our Method.}
Given data to an unknown manifold, our goal is to construct a global reference frame defined by a trainable center $C$.
(a) Learning a single unconstrained vector field is insufficient: such a field can always be rescaled to collapse the entire manifold onto $C$.
(b) Imposing a non-contracting constraint—by requiring the Lie derivative of the metric to be positive semi-definite (see Appendix~\ref{app:background})—prevents this collapse and enforces geometric consistency.
(c) By sequentially combining $m$ tangent vector fields (matching the intrinsic dimension of the manifold), we can transport every point to a unique location associated with the global frame’s center while ensuring trajectories remain on the manifold.}
\label{fig::main_method}
\end{figure*}

While these properties characterize the desired representation in the exact setting, this characterization alone is not sufficient for learning from finite data. In practice, the corresponding constraints are enforced only approximately, and small empirical residuals may coexist with highly distorted or poorly conditioned representations. In particular, nearby variations can be strongly compressed while the coordinate equations remain approximately satisfied on the observed samples; see Figure~\ref{fig::main_method_a}.

We therefore introduce \emph{non-contraction}, see Figure~\ref{fig::main_method_b}, an additional desirable property of the learned representation. Unlike commutativity, time coherence, and common-reference reachability, non-contraction is not necessary for a map to define valid global coordinates: a diffeomorphic parametrization may strongly contract some regions while remaining globally valid. The property is desirable for a different reason. When the objective is to uncover the geometry of the data manifold, applying a coordinate transformation should not erase variations transverse to the active feature. Returning to the face example, changing the smile of two faces should not artificially suppress differences in pose, identity, or other factors. Non-contraction therefore complements the structural properties of global coordinates with a geometric preservation property: it favors flows that retain the local variability of the data and rules out approximate low-loss solutions that satisfy the coordinate relations by collapsing part of the underlying manifold geometry. 

Our framework separates two complementary roles: commutativity, time coherence, and a common reference define the global flow-based coordinates, while non-contraction preserves geometric variations. The first three properties characterize the targeted flow-based coordinates in the exact setting, while non-contraction makes their approximate recovery suitable for geometry discovery. This leads to the learning principle
\[
\boxed{
\text{Commutativity + Time coherence + Common reference + Non-contraction}
\;\stackrel{\mathrm{approx.}}{\Longrightarrow}\;
\text{Global coordinates.}
}
\]
This formulation therefore links globally consistent coordinates with preservation of the underlying manifold geometry within a single learning objective.

Our contributions can be summarized as follows:
\begin{itemize}
\item \textbf{A unified geometric view for globally consistent representations.} We bring together commutativity, time coherence, and common-reference reachability under the classical differential-geometric structure of global coordinates, and show how these properties characterize the targeted global flow-based representation. This provides a unified perspective for representation learning, where related notions have largely been considered separately.

\item \textbf{Property-driven manifold geometry learning.}
Building on this characterization, we introduce an unsupervised framework that learns a representation from the properties it should satisfy, rather than taking the representation map itself as the primary learning object and assessing such properties only \emph{a posteriori}. We introduce non-contraction as an additional geometric preservation principle required to recover the targeted coordinate structure under approximate learning and to identify the tangent geometry of the data manifold.

\item \textbf{Scalable flow-based geometry discovery.}
We derive a scalable, integration-free training objective that exploits the structure of commuting flows, enabling globally consistent coordinate directions and tangent geometry to be learned directly from high-dimensional sampled data.

\end{itemize}

\section{Related work}

\paragraph{Unconstrained and disentangled representations.}
Many representation-learning approaches map data lying near an unknown
low-dimensional manifold to a latent space without explicitly constraining
the geometry of this representation. This includes autoencoders and
variational autoencoders~\citep{kingma2014vae,pmlr-v32-rezende14},
adversarially learned representations such as BiGAN and
ALI~\citep{donahue2017adversarial,dumoulin2017adversarially}, and
contrastive or self-supervised approaches based primarily on invariance to
data augmentations~\citep{chen2020simclr,he2020moco,zbontar2021barlow}.
Disentanglement further seeks latent variables associated with distinct
factors of variation~\citep{higgins2017betavae,pmlr-v80-kim18b}, although
such factors cannot in general be identified without additional inductive
biases~\citep{pmlr-v97-locatello19a}. These approaches primarily constrain
the representation through reconstruction, distributional, or statistical
criteria and do not explicitly require the coordinate-flow structure
considered in this work.

\paragraph{Structural properties for consistent representations.}
Several properties related to our formulation have appeared independently in representation learning. Commutativity has been used to obtain order-independent transformations: \citet{Aoshima_2023_CVPR} construct commuting nonlinear vector fields for semantic editing in the latent space of a pretrained generative model, while \citet{zhu2021commutative} use a commutative Lie-group decomposition to represent disentangled variations. More generally, symmetry-based definitions of disentanglement describe transformations that affect specific components of a representation while leaving others unchanged~\citep{higgins2018towards}. Related ideas arise through canonicalization and common-reference constructions. \citet{pmlr-v202-kaba23a} learn transformations mapping samples to canonical representatives of their symmetry orbits, while CaDeX~\citep{Lei_2022_CVPR} uses continuous bijective maps through a learned canonical shape to represent non-rigid deformations in a shared coordinate space. Such constructions organize transformations relative to a reference, although they do not formulate the single flow-reachable common reference considered here. Taken together, these works establish commutativity, transformation consistency, and canonicalization as useful representation principles. Our formulation brings these properties under a common coordinate-flow structure and, crucially, uses them jointly as learning primitives from which the global manifold representation is recovered.

\paragraph{Metric-aware manifold representations.}
Metric-based approaches preserve or recover intrinsic geometry through approximate isometries or learned pullback metrics~\citep{Shao2017TheRG,diepeveen2025manifoldlearningnormalizingflows,kruiff2025pullback}. In particular, isometric representation learning~\citep{hahm2024isometric} prevents contraction through distance preservation, whereas our non-contraction condition directly constrains how infinitesimal distances evolve along the learned flows. Curvature-aware methods instead model higher-order geometric structure, for instance through Ricci-type objectives~\citep{lazarev2025ricci}. Our formulation therefore treats non-contraction as a distinct geometric principle: learned flows may deform distances, but they should not erase transverse geometric variations.

\paragraph{Implicit and topology-inspired manifold representations.}
Other approaches describe manifolds without constructing explicit global
coordinates. Implicit neural, score-based, and energy-based methods
characterize the data manifold through level sets, scores, or energy
landscapes~\citep{LeCun2006ATO,Song2020ScoreBasedGM,bethune2025follow,
pmlr-v267-diepeveen25a}. Such methods are well suited to projection or
sampling near complex data distributions, but do not directly provide
coordinate flows or an explicit tangent frame. Bundle
Networks~\citep{courts2022bundle} instead exploit local fiber-bundle
factorizations for conditional generative modeling. While this introduces
explicit differential-topological structure, the objective is local
base--fiber decomposition rather than recovery of globally consistent
intrinsic directions.

\paragraph{Contraction and non-contraction of learned flows.}
Contraction theory studies the evolution of infinitesimal distances between
neighboring trajectories and typically imposes a negative-semidefinite
differential metric condition to guarantee convergence and
stability~\citep{NEURIPS2019_2647c1db,pmlr-v162-rodriguez22a}. Our use of
the same differential object has the opposite purpose. We impose a
positive-semidefinite condition on the Lie derivative of the Euclidean
metric along the learned coordinate flows, preventing neighboring
trajectories from becoming closer. Non-contraction is therefore not part of
the classical characterization of global coordinates, but an additional
desirable geometric property introduced for their approximate recovery:
it prevents coordinate-consistency losses from being reduced by erasing
transverse variation of the data manifold.

\paragraph{Positioning through representation properties.}
Table~\ref{tab:property-comparison} positions representative approaches according to both the geometric properties they target and the role these properties play in learning. In most existing methods, a representation map is the primary object and geometric properties are imposed on, induced by, or evaluated from that map. Our formulation reverses this perspective: the structural properties are taken as the learning primitives, and the global coordinate representation is identified as their consequence. The table therefore distinguishes properties that are explicitly targeted from related or implicit notions, while also highlighting this difference in learning principle.

\begin{table}[h]
\centering
\small
\renewcommand{\arraystretch}{1.25}
\setlength{\tabcolsep}{4.5pt}
\begin{tabular}{lccccc}
\toprule
\textbf{Method}
&
\textbf{Comm.}
&
\textbf{Time coh.}
&
\textbf{Common ref.}
&
\textbf{Non-contract.}
&
\textbf{Primitive.}
\\
\midrule

\textbf{Ours}
& \checkmark
& \checkmark
& \checkmark
& \checkmark
& Properties
\\

DeCurvEd~\citep{Aoshima_2023_CVPR}
& \checkmark
& $\checkmark$
& --
& --
& Map
\\

Comm. Lie Group~\citep{zhu2021commutative}
& \checkmark
& $\sim$
& --
& --
& Map
\\

Canonicalization~\citep{pmlr-v202-kaba23a}
& --
& --
& $\sim$
& --
& Map
\\

CaDeX~\citep{Lei_2022_CVPR}
& --
& --
& $\sim$
& --
& Map
\\

Isometric representations
& --
& --
& --
& $\sim$
& Map
\\

Unconstrained AE
& --
& --
& --
& --
& Map
\\

\bottomrule
\end{tabular}

\vspace{0.1cm}

\caption{
Relation between representative representation-learning approaches and
the four properties considered in this work. \(\checkmark\) denotes a property explicitly targeted by the method, \(\sim\) a related notion, and -- a property that is not targeted. The last column indicates whether the representation is primarily defined from a map or identified from its structural properties.
}
\label{tab:property-comparison}
\end{table}

\section{Method}

\subsection{Property-driven formulation}
\label{sec:method}
\paragraph{Main idea.}
Our method is designed to recover a global coordinate representation of an unknown data manifold by learning vector fields in the ambient Euclidean space whose restrictions parameterize the manifold’s tangent spaces. These vector fields induce flows that transport data points along intrinsic directions of the manifold toward a common global reference point, thereby yielding an interpretable global coordinate system embedded in the ambient space.

\paragraph{Setting and objective.}
Let \(M\subseteq\mathbb{R}^n\) be an unknown \(m\)-dimensional data manifold observed through samples. Our objective is to recover a global coordinate system of \(M\) together with its tangent directions. Following the property-driven perspective, we characterize the desired representation through the behavior of flows on the manifold rather than by prescribing a coordinate map directly.

\begin{definition}[Flow]
\label{def::flow}
Let \(F:\mathbb{R}^n\to\mathbb{R}^n\) be a vector field. We denote by
\(\Phi_F^t\) its flow, defined by
\[
\Phi_F^t(x)=z(t),
\qquad
\dot z(t)=F(z(t)),
\qquad
z(0)=x.
\]
Given a vector field \(F_i\) and an associated time function \(T_i:\mathbb{R}^n\to\mathbb{R}_{+}\), we
define
\[
\phi_i(x):=\Phi_{F_i}^{\,T_i(x)}(x).
\]
\end{definition}

Following the canonical form for commuting vector fields~\citep[Theorem~9.46]{lee_introduction_2012}, we call \emph{global flow-based coordinates} a global coordinate $\{T_i\}$ of the manifold obtained by composing vector-field flows $\{F_i\}$ from a common reference point $C$. This global coordinate structure targeted by our method is characterized by three complementary properties of these flows, stated in Theorem~\ref{thm:global_coordinates_coherence}. 

\begin{theorem}[Characterization of global flow-based coordinates]
\label{thm:global_coordinates_coherence}
Let $\mathcal M\subseteq\mathbb R^n$ be an $m$-dimensional connected
manifold, let $T=(T_1,\ldots,T_m):\mathcal M\rightarrow\mathbb R^m$ be $C^2$, and let $F_1,\ldots,F_m$ be $C^1$ tangent vector fields with associated flows $\Phi_i^t$. \\

Then $T$ and $F$ define global flow-based coordinates with a common reference point $C\in\mathcal M$, if and only if the
following properties hold:

\emph{Commutativity -- vector-field invariance:}
\[
[F_i,F_j]=0,
\qquad
\forall i,j\in\{1,\ldots,m\},
\]

\emph{Time coherence -- coordinate invariance:}
\[
\nabla T_j(x)^\top F_i(x)=-\delta_{ij},
\qquad
\forall i,j\in\{1,\ldots,m\},
\]

\emph{Common reference:}
\[
\Phi_m^{T_m(x)}
\circ\cdots\circ
\Phi_1^{T_1(x)}(x)=C,
\qquad
\forall x\in\mathcal M.
\]
\end{theorem}

Theorem~\ref{thm:global_coordinates_coherence} therefore motivates learning the representation by directly optimizing commutativity, time coherence, and convergence to a common reference point. These three properties form the structural objective of our method. Proof of this theorem can be found in Appendix~\ref{app:characterization_global_flow_coordinates}

\paragraph{Degeneracy under approximate learning.}In practice, however, the manifold is observed only through finitely many samples and these identities can only be enforced approximately. Small violations of the structural conditions do not necessarily imply that the learned flows preserve the geometry of the manifold. In particular, the common-reference objective admits a degenerate mechanism: the learned vector fields may strongly contract neighboring trajectories so that all samples are easily transported toward the same reference point. Such a solution can reduce the endpoint error while progressively erasing variations transverse to the direction of the flow. The resulting representation may therefore satisfy the coordinate constraints approximately while providing a poor description of the tangent geometry of the data manifold. This behavior is illustrated in Figure~\ref{fig::main_method_a}.

\paragraph{Non-contracting property.} To rule out this degeneracy, we introduce \emph{non-contracting} property as an additional geometric property of the learned flows. For each vector field \(F_i\), we impose the positive-semidefiniteness of the Lie derivative of a conformally Euclidean metric defined by $\sigma_i : \mathbb{R}^n \to \mathbb{R}$, a bounded smooth scalar field: \[ g_i(x)=e^{\sigma_i(x)}I_n, \qquad \mathcal{L}_{F_i}g_i\succeq0. \] This condition prevents infinitesimal distances from decreasing along the flow and therefore prevents the learned representation from erasing transverse manifold variations (Figure~\ref{fig::main_method_b}).

\subsection{Theoretical Results}

The global parametrization considered in this work applies to manifolds satisfying a mild regularity condition, formalized through the notion of \emph{admissible manifolds}, namely compact embedded submanifolds of $\mathbb{R}^n$ that admit a global chart onto an Euclidean domain. This assumption ensures the absence of topological obstructions and provides a natural target for the coordinate functions produced by our method. \\

\begin{definition}[Admissible Manifolds with a Global Chart]
\label{def::manifold}
Let \(\mathcal{M}_m\) denote the class of all compact, connected, smoothly embedded \(m\)-dimensional submanifolds with boundary of \(\mathbb{R}^n\) that admit a \emph{global chart} in the sense of manifolds with boundary:

$
\qquad \qquad \psi: M \xrightarrow{\;\cong\;} U \subset \mathbb{R}^m \qquad
$
where $M \in \mathcal{M}_m$ and \(U\) is a compact set with nonempty interior.
\end{definition}

We restrict the analysis to the admissible flows introduced above, satisfying the time-coherence and non-contraction conditions.

\begin{definition}[Admissible flow maps] 
\label{def::flows_metric_times} 
Let $\mathcal F$ denote the class of flow maps $\phi(x)=\Phi_F^{T(x)}(x)$ , whose associated $(F,T,\sigma)$ satisfy the time-coherence and non-contraction conditions, namely \[ \big\langle \nabla T(x),\, F(x) \big\rangle = -1 \quad \forall x \in M \in \mathcal{M}_m, \quad \mathcal{L}_F g  \succeq 0 \quad \text{on } \mathbb{R}^n , \quad g(x)=e^{\sigma(x)}I_n, \quad  \big| \langle \nabla \sigma(x), F(x) \rangle \big| \le K < \infty. \] 
\end{definition}

\begin{definition}[Compatible families of admissible flows]
\label{def:compatible_flows}
For $k\geq1$, let $\mathcal K_k\subset\mathcal F^k$ denote the class of
families $(\phi_1,\ldots,\phi_k)$, with associated
$T=(T_1,\ldots,T_k)$ and $F=(F_1,\ldots,F_k),$ such that
$$ JT(x)F(x)=-I_k, \qquad \forall x\in M, [F_i,F_j]=0, \qquad \forall\,i\neq j.$$

\end{definition}

The flows are then composed sequentially. Each vector field spans a one-dimensional subbundle of the tangent bundle of the manifold, and the length traveled along its integral curves defines a scalar quantity $\ell_i(x)$. These quantities are interpreted as intrinsic coordinates associated with the successive flow directions, and their joint behavior under composition is central to our theoretical analysis. \\

\begin{definition}[Lengths of field line segments]
Given a sequence of flows $\{\phi_1, \dots, \phi_k\}$ generated by integrating $F_i$ during the time $T_i(x)$, and a point $x \in \mathbb{R}^n$, the quantity $\ell_i(x)$ denotes the length of the field line segment traced by $F_i$ during time $T_i$, starting from
\[
(\phi_{i-1} \circ \cdots \circ \phi_1)(x).
\]
\end{definition}
A formal definition of $l_i$ is provided in Appendix~\ref{app:background}.

Theorem~\ref{thm:loss_known} establishes an approximation guarantee for the proposed objective under the admissibility conditions introduced above. Specifically, arbitrarily small loss entails arbitrarily small deviation of the learned vector fields from the tangent bundle, rather than only characterizing the exact zero-loss solution. For a known intrinsic dimension \(m\), this provides an approximation guarantee and, in the zero-loss limit, recovers the global result that the induced flows cover \(M\) and the associated intrinsic lengths \((\ell_1,\ldots,\ell_m)\) define a global chart.

\begin{theorem}[Loss Minimization when the Dimension is Known]
\label{thm:loss_known}
Let \(M\in\mathcal M_m\), let \((\phi_1,\ldots,\phi_m)\in\mathcal K_m\), and let \(X\sim\rho\) with \(\operatorname{supp}(\rho)=M\). For every \(\delta>0\), there exists \(\varepsilon>0\) such that:
$$
\mathcal L_m
=
\mathbb E_{x\sim X}
\left[
\left\|
C-(\phi_m\circ\cdots\circ\phi_1)(x)
\right\|_2^2
\right] <\varepsilon
\quad\Longrightarrow\quad
\max_{i=1,\ldots,m}\sup_{x\in M}
\frac{\|P_N(x)F_i(x)\|_2}{\|F_i(x)\|_2}
<\delta ,
$$

where \(P_N(x)\) denotes the orthogonal projector onto the normal space of \(M\).

In particular, as \(\mathcal L_m\to0\), the learned vector fields converge to tangency with \(M\). At zero loss, their flows cover \(M\), and the associated functions \((\ell_1,\ldots,\ell_m)\) define a system of global coordinates on \(M\).
\end{theorem}

\paragraph{Geometric interpretation of the learned map.}
Geometrically, each \(T_i(x)\) measures the flow time required to transport \(x\) along \(F_i\) to the reference level set \(T_i=0\). Indeed, time coherence gives \(T_i(\Phi_i^t(x))=T_i(x)-t\), while the other time coordinates remain unchanged. Collecting these flow times therefore defines a coordinate map which locates each point by its displacement, along the learned flow directions, from the corresponding reference sections. The associated length map \(\ell=(\ell_1,\ldots,\ell_m)\) measures instead the intrinsic arc lengths traveled along the same flow curves. Flow time and arc length differ only by a smooth reparametrization along these curves, so \(T\) provides a dynamical representation of the intrinsic coordinate map \(\ell\), and the two coincide for unit-speed fields. This relation is formalized in Appendix~\ref{app:time_length_coordinates}.

When the intrinsic dimension is unknown, the loss provides a criterion for assessing whether the dimension is sufficient. Theorem~\ref{thm:loss_unknown} formalizes this property by showing that, for $k<m$, the loss remains bounded away from zero, whereas Theorem~\ref{thm:loss_known} establishes that arbitrarily small loss is attainable at the intrinsic dimension. Although zero loss is not attained numerically, this yields a desirable dimensional-consistency property, consistent with intrinsic-dimension estimation based on representation error~\citep{Zeng_2024}.

\begin{theorem}[Main Characterization via Loss Minimization]
\label{thm:loss_unknown}
Let $M \in \mathcal{M}_m$, and suppose $k < m$. Let \((\phi_1, ..., \phi_m) \in \mathcal{K}_m\) introduced in Definition~\ref{def:compatible_flows}. Let $X \sim \rho$ be a random variable sampled 
from the distribution supported on $M$. There exists $c > 0$ such that $\mathcal{L}_k \geq c$ where:
\[
\mathcal{L}_k = \min_{\{T_i,F_i, \sigma_i\}_{i=1}^k, \, C} 
\mathbb{E}_{x \sim X} \Big[ \|\, C - (\phi_k \circ \cdots \circ \phi_1)(x) \,\|^2_2 \Big]
\]
\end{theorem}
Proofs of Theorems~\ref{thm:loss_known} and~\ref{thm:loss_unknown} are provided in Appendices~\ref{app:proof_loss_known} and~\ref{app:proof_loss_unknown}, respectively. The auxiliary lemmata supporting these proofs are collected in Appendix~\ref{app:lemmata_for_loss_known}.

\begin{minipage}[t]{0.55\textwidth}
\vspace{0pt}
\paragraph{Commutativity-induced flow.}
Evaluating the composite flow
$\phi_m\circ\cdots\circ\phi_1$ requires integrating $m$ Neural ODEs
sequentially, which becomes increasingly expensive as $m$ grows.
Commutativity allows this composition to be evaluated as a single flow.
As illustrated in Figure~\ref{fig::commuting_flows}, flowing first along
$F_1$ for time $T_1$ and then along $F_2$ for time $T_2$ reaches the same
point as integrating the combined field
$T_1F_1+T_2F_2$ over a unit time interval.
Hence, the sequence of $m$ Neural ODEs can be replaced by the integration
of a single vector field without changing the resulting transport.
Theorem~\ref{thm:equiv_commuting} formalizes this property and reformulates
the frame-learning loss of Theorem~\ref{thm:loss_known} as a single ODE.
\end{minipage}
\hfill
\begin{minipage}[t]{0.41\textwidth}
\vspace{0pt}
\centering

\begin{tikzpicture}[>=stealth, scale=1.7]

\begin{scope}
  \clip (-1.6,0) rectangle (1.6,1.6);

  \shade[ball color=gray!15, opacity=0.25] (0,0) circle (1.6);

  \foreach \angle in {-160,-120,-90,-60,-40,-20,0,20,40,60}{
    \draw[gray!50, thin, opacity=0.5]
      plot[domain=0:180, samples=80]
      ({1.6*sin(\x)*cos(\angle)}, {1.6*cos(\x)});
  }

  \foreach \h in {-1.2,-0.8,-0.4,0,0.4}{
    \draw[gray!50, thin, opacity=0.5]
      ({sqrt(1-(\h/1.6)^2)*1.6},{\h})
      arc[
        start angle=0,
        end angle=180,
        x radius={sqrt(1-(\h/1.6)^2)*1.6},
        y radius={0.8}
      ];
  }
\end{scope}

\coordinate (x)     at (-0.4,0.4);
\coordinate (phi1x) at (0.6,0.40);
\coordinate (C)     at (0.6,0.85);

\draw[<-, very thick, blue]
  (phi1x)
  .. controls (0.25,0.25) and (0.,0.2) ..
  (x)
  node[below] at (0.15,0.3) {$T_1F_1$};

\draw[->, very thick, orange]
  (phi1x)
  .. controls (0.65,0.45) and (0.65,0.65) ..
  (C)
  node[pos=0.55,right] {$T_2F_2$};

\draw[->, very thick, green!60!black]
  (x)
  .. controls (0.,0.45) and (0.3,0.5) ..
  (C)
  node[left] at (0.2,0.65) {$T_1F_1+T_2F_2$};

\fill (x)     circle (0.035) node[below left] {$x$};
\fill (phi1x) circle (0.035) node[below] {$\phi_1(x)$};
\fill (C)     circle (0.04) node[above right]
      {$\phi_2(\phi_1(x))$};

\end{tikzpicture}

\captionof{figure}{
\textbf{Commutativity-induced flow simplification.}
Sequential flows generated by $(F_1,T_1)$ and $(F_2,T_2)$ are equivalent
to integrating their combined vector field.}
\label{fig::commuting_flows}

\end{minipage}

\vspace{0.2cm}
\begin{theorem}[Equivalent Loss under Commuting Flows]
\label{thm:equiv_commuting}
Let $T:\mathbb{R}^n\rightarrow\mathbb{R}_{+}^m$ be defined componentwise by $T(x)=(T_1(x), ..., T_m(x))$, let $F:\mathbb{R}^n\rightarrow\mathbb{R}^{n\times m}$ be defined as $F(x)=(F_1(x), ..., F_m(x))$ and let $\sigma:\mathbb{R}^n\rightarrow\mathbb{R}^{m}$ be defined as $\sigma(x)=(\sigma_1(x), ..., \sigma_m(x))$.   Then the loss in Theorem~\ref{thm:loss_known} can be rewritten as:
\[
\mathcal{L}_m 
=
\min_{T, F, \sigma, \, C} 
\mathbb{E}_{x \sim X}\! \lim_{t \rightarrow \infty} \left[
  \big\|\, C - \phi^{t}_{\sum T_i F_i}(x) \,\big\|_2^2
\right],
\]
\end{theorem}
where $\phi$ denotes the flow defined in Definition~\ref{def::flow}. The proof of the theorem is provided in Appendix~\ref{app:proof_equiv_commute}.

\paragraph{Integration-free training via path-wise convergence.}
Replacing the sequential integration of $m$ flows by a single autonomous Neural ODE reduces computational burden. Yet this ODE must still be evolved to equilibrium, formally requiring infinite time and causing numerical difficulties. To remove this dependence, we introduce, inspired by flow-matching methods~\citep{metric_flow_matching}, a curve \(u_s\) connecting \(x\) to \(C\) and match its velocity to that induced by the autonomous vector field. Theorem~\ref{thm:integration_free_loss} formalizes this construction as an integration-free surrogate objective. Since \(u_s\) indexes position along the path rather than physical time, training reduces to local velocity matching without integration to equilibrium (proof in Appendix~\ref{app:proof_integration_free_loss}).

\vspace{0.2cm}

\begin{theorem}[Integration-free Loss]
\label{thm:integration_free_loss}
Let $T:\mathbb{R}^n\rightarrow\mathbb{R}_{+}^m$ be defined componentwise by $T(x)=(T_1(x), ..., T_m(x))$, let $F:\mathbb{R}^n\rightarrow\mathbb{R}^{n\times m}$ be defined as $F(x)=(F_1(x), ..., F_m(x))$ and let $\sigma:\mathbb{R}^n\rightarrow\mathbb{R}^{m}$ be defined as $\sigma(x)=(\sigma_1(x), ..., \sigma_m(x))$. Let $s:\mathbb{R}^n\rightarrow\mathbb{R}^n$ and $u_s:[0,1]\times\mathbb{R}^n\to\mathbb{R}^n$ be an interpolating curve such that $u_s(0,x)=x$ and $u_s(1,x)= C$.

Then the objective in Theorem~\ref{thm:equiv_commuting} admits the following integration-free surrogate:
\[
  \min_{T,F,\sigma, s, C}\; \mathbb{E}_{x\sim X}\,
\mathbb{E}_{t \sim \mathbb{U}_{[0,1]}} \Bigl\| F\!\big(u_s(t,x, C)\big)\, T\!\big(x\big)
  -\partial_t u_s(t,x,C) \Bigr\|_2^2.
\]
\end{theorem}

\begin{remark}
The objective in Theorem~\ref{thm:integration_free_loss} approximates the corresponding integral loss by sampling \(t\). Similar to flow-matching methods~\citep{lipman2023flow}, finite sampling may yield localized inaccuracies in regions of rapidly varying dynamics. No evidence of such behavior emerged in the reported experiments.
\end{remark}

Although Theorem~\ref{thm:integration_free_loss} only requires \(u_s\) to satisfy the prescribed boundary conditions, we found empirically that imposing additional structure significantly improves training. At the optimum, \(u_s\) is intended to reproduce trajectories of the autonomous vector field and should therefore inherit three corresponding properties: it connects each \(x\) to the common endpoint \(C\); distinct trajectories cannot cross before reaching \(C\), as follows from uniqueness of the flow; and moving successively along the same trajectory should be equivalent to moving directly to the resulting position. Proposition~\ref{prop:us_parametrization} enforces these properties by introducing a diffeomorphic coordinate system in which trajectories become straight radial contractions.

\begin{proposition}[$u_s$ parametrization]
\label{prop:us_parametrization}
Let $s:\mathbb R^n\to\mathbb R^n$ be a $C^1$-diffeomorphism and set $C=s^{-1}(0)$. For $x\in\mathbb R^n$ and $t\in[0,1]$, define
\[
u_s(t,x)=s^{-1}\big((1-t)s(x)\big).
\]
Then the following properties hold:
\begin{enumerate}
    \item \textbf{Boundaries.} $u_s(0,x)=x$ and $u_s(1,x)= C$
    \item \textbf{Non-crossing.} For every $t<1$, the map $x\mapsto u_s(t,x)$ is injective. Equivalently, distinct radial orbits in $s$-coordinates do not cross before reaching $c$.

    \item \textbf{Time coherence.} For every $s,t\in[0,1)$,
    \[
    u_s(t,u_s(s,x))=u_s(t+s-ts,x).
    \]
\end{enumerate}
\end{proposition}

\paragraph{Design of $T$ and $F$.} Rather than enforcing time coherence and commutativity through additional regularization terms, we encode both properties directly in the parametrization of \(T\) and \(F\). Proposition~\ref{prop:TF_parametrization} defines this construction from the same diffeomorphism \(s\) used to generate the exploration map \(u_s\).

\begin{proposition}[Compatible parametrization of $T$ and $F$]
\label{prop:TF_parametrization}
Let $s:\mathbb R^n\to\mathbb R^n$ be the $C^1$-diffeomorphism introduced
in Proposition~\ref{prop:us_parametrization}, and let
\[
P=
\begin{bmatrix}
I_m & 0_{m\times(n-m)}
\end{bmatrix}
\in\mathbb{R}^{m\times n}.
\]

Define $\quad T(x)=P \; s(x), \quad F(x)=-J_{s^{-1}}(s(x))P^\top,$
where the $i$-th column of $F$ defines the vector field $F_i$.

Then:
\begin{enumerate}
    \item \textbf{Time coherence.}
    \[
    JT(x)F(x)=-I_m.
    \]

    \item \textbf{Commutativity.}
    \[
    [F_i,F_j](x)=0,
    \qquad \forall\, i \neq j.
    \]

    \item \textbf{Compatibility with $u_s$.}
    Along the trajectories, the time coordinates satisfy
    \[
    T(u_s(t,x))=(1-t)T(x).
    \]
\end{enumerate}
\end{proposition}
Proofs of Proposition~\ref{prop:us_parametrization} and Proposition~\ref{prop:TF_parametrization} are provided in the Appendix~\ref{app:proof_us_parametrization} and Appendix~\ref{app:proof_us_TF_parametrization}.

\paragraph{Theoretical example.}  A fully analytical example is provided in the Appendix~\ref{app:theoretical_example}. For the helicoidal manifold, we explicitly derive the trajectory map \(u_s\), the time functions \(T\), the vector fields \(F\).

\subsection{Training Formulation}
    
\paragraph{Enforcing positive semidefiniteness of the Lie derivative.} To enforce the positive semidefiniteness (PSD) of the Lie derivative, we seek a regularization whose magnitude directly reflects the geometric effect of violating this constraint. More precisely, small violations of the PSD condition should imply only small transverse contraction of infinitesimal perturbations along the flow. This ensures that the regularizer is not merely algebraic, but quantitatively linked to the non-contracting behavior that the model is designed to enforce. The following proposition formalizes this connection:
\begin{equation} 
\label{loss} 
\begin{aligned}
    \mathcal{L}(x,t) &= \left\| \sum_iT_i(x)F_i(z_t)-\partial_t u_s(t,x) \right\|^2 + \frac{\alpha}{4} \sum_i \|T_i(x)F_i(z_t)\|^2  \sum_k \left( (- T_i(x) \lambda_{i,k}(z_t)) \right)_{+}^2, \\
    &\text{where } \lambda_{i,k}(z_t) = \lambda_k(S_i(z_t)) \quad \text{and }  S_i(z_t) = JF_i(z_t)+JF_i(z_t)^\top.
\end{aligned}
\end{equation}
The geometric justification of this weighting is established in Proposition~\ref{prop:loss}, while Proposition~\ref{prop:conformal_reduction} shows that the conformal factor \(\sigma_i\) can be eliminated through an equivalent reduced formulation, with its contribution represented through the coefficient \(\alpha\). Both results are provided in the Appendix. An empirical sensitivity analysis of the non-contraction weight $\alpha$ is reported in Appendix~\ref{app:alpha_ablation}.

\paragraph{Training procedure.}
Algorithm~\ref{alg:loss} summarizes the practical implementation of our proposed formulation. For each training sample, a point is first generated along the trajectory induced by the bijective map $s$, after which the associated fields $F$ and time maps $T$ are obtained from the compatible parametrization of Proposition~\ref{prop:TF_parametrization}. The training loss in Equation~\ref{loss} is then evaluated at this trajectory point, combining the speed-matching objective with the non-contraction regularization. The parameters of $s$ are finally updated by backpropagation over the resulting batch loss.

\begin{algorithm}[ht]
\caption{Training loss evaluation and update}
\label{alg:loss}
\begin{algorithmic}[1]

\Require Batch $\{x^{(b)}\}_{b=1}^B$, bijective map $s$

\For{$b=1$ \textbf{to} $B$}

    \State Sample $t\sim\mathcal{U}[0,1]$

    \State Compute
    $z\gets u_s(t,x^{(b)},C)$
    (defined in Proposition~\ref{prop:us_parametrization})

    \State Compute the functions $F(x)$, $T(x)$ (defined in Proposition~\ref{prop:TF_parametrization})
    
    \State Compute the loss $\mathcal{L}(x,t)$ (defined in Equation~\ref{loss})

\EndFor

\State Update the parameters of $s$ by backpropagation
\end{algorithmic}
\end{algorithm}

\section{Experiments} 

The experiments are designed to assess whether the proposed geometric structure is recovered in accordance with the theoretical framework. We first examine the role of non-contraction in intrinsic-dimension identification, and then test our formulation without imposing the specific parametrization of \(F\), \(T\), and \(u_s\). We subsequently evaluate the proposed architecture on tangent recovery and global consistency, using an unconstrained autoencoder to assess whether accurate reconstruction and local geometry are sufficient for these global properties to emerge. Finally, we study the behavior of the method beyond the global-coordinate assumption. Experimental details are provided in Appendix~\ref{app:exp_setup}, while dataset definitions are given in Appendix~\ref{app:datasets}. For each dataset, the pair $(a,b)$ denotes the intrinsic and ambient dimensions, respectively.

\paragraph{Empirical validation of the theoretical results.}
We first evaluate the dimension characterization of Theorems~\ref{thm:loss_known} and~\ref{thm:loss_unknown} on a uniformly sampled three-dimensional linear manifold embedded in \(\mathbb R^4\). Models with \(k\in\{1,2,3,4\}\) vector fields are trained with and without the non-contraction regularization, allowing us to assess whether vanishing loss identifies the intrinsic dimension and whether non-contraction prevents degenerate solutions. As reported in Table~\ref{tab:linear-manifold-loss}, the observed behavior follows these theoretical predictions, with the loss identifying the correct intrinsic dimension only when non-contraction is enforced. We next evaluate on three synthetic datasets Theorem~\ref{thm:integration_free_loss} independently of the specific parametrizations of \(F\), \(T\), and \(u_s\) introduced in Propositions~\ref{prop:us_parametrization} and~\ref{prop:TF_parametrization}. In this experiment, \(F\) and \(T\) are learned through regularization, while the exploration path is parameterized as \(u_s(t,x,C)=(1-t)x+tC+t(1-t)s(t,x)\), inspired by flow matching. As reported in Table~\ref{table_traing_regu_based}, the learned directions achieve small angular errors, substantially below the random-direction baseline, supporting the tangent-space recovery predicted by the theorem independently of the architectural parametrization. Implementation details are provided in Appendices~\ref{app::linear_exp} and~\ref{app::without_arch_param}.

\begin{table*}[t]
\centering

\begin{subtable}[c]{0.48\linewidth}
\centering

\begin{tikzpicture}[remember picture,overlay]
\node[
anchor=north west,
font=\bfseries
] at (-4,0.7) {(a)};
\end{tikzpicture}

\begin{small}
\begin{tabular}{ccc}
\toprule
\textbf{\# vector fields} & \textbf{Loss} & \textbf{Non-contraction} \\
\midrule
1 & $0.042$ & \checkmark \\
2 & $0.015$ & \checkmark \\
3 & $3.92,10^{-6}$ & \checkmark \\
4 & $1.22,10^{-6}$ & \checkmark \\
1 & $1.7,10^{-8}$ & -- \\
\bottomrule
\end{tabular}
\end{small}

\phantomsubcaption
\label{tab:linear-manifold-loss}
\end{subtable}
\hfill
\begin{subtable}[c]{0.48\linewidth}
\centering

\begin{tikzpicture}[remember picture,overlay]
\node[
anchor=north west,
font=\bfseries
] at (-4,0.85) {(b)};
\end{tikzpicture}

\begin{small}
\begin{tabular}{lcc}
\toprule
& \multicolumn{2}{c}{\textbf{Manifold Tangent error} $\downarrow$} \\
\cmidrule(lr){2-3}
\textbf{Dataset}
& \textbf{Our Regularize}
& \textbf{Random} \\
\midrule
Paraboloid $(2,3)$ & $0.55^\circ \pm 0.07$ & $30^\circ$ \\
Swiss Roll $(2,3)$ & $3.11^\circ \pm 1.05$ & $30^\circ$ \\
Sphere $(2,3)$     & $0.97^\circ \pm 0.11$ & $30^\circ$ \\
\bottomrule
\end{tabular}
\end{small}

\phantomsubcaption
\label{table_traing_regu_based}
\end{subtable}

\caption{
\textbf{Empirical validation of the theoretical results.}
\textbf{(a)} Loss associated with Theorems~\ref{thm:loss_known} and~\ref{thm:loss_unknown} on a 3-D linear manifold in $\mathbb{R}^4$, as a function of the number of learned vector fields. With non-contraction, the loss vanishes only from three vector fields onward, correctly identifying the intrinsic dimension. In contrast, without non-contraction, a vanishing loss is already obtained with a single field, which no longer reflects the manifold dimension and reveals the degenerate solution.\textbf{(b)} Angular deviation of the learned vector-field directions from the ground-truth manifold tangent space, as defined in Appendix~\ref{app:tangent_space_metrics}. Here, \(F\), \(T\), and \(u_s\) are learned using the loss of Theorem~\ref{thm:integration_free_loss}, without the parametrizations of Propositions~\ref{prop:us_parametrization} and~\ref{prop:TF_parametrization}. Smaller angles indicate better tangent alignment. The learned directions remain well below the random-direction baseline \ref{app:tangent_space_metrics}, showing that the manifold tangent space is recovered independently of our parametrization.}
\end{table*}

\paragraph{Tangent recovery on synthetic manifolds.}
We first evaluate local geometric recovery on synthetic manifolds for which the analytical tangent spaces are known, using the angular discrepancy between estimated and ground-truth tangent spaces. An unconstrained autoencoder (AE) provides a map-first baseline against which to compare our property-driven formulation. The first two columns of Table~\ref{table_toys} show that both approaches recover the local tangent geometry accurately across datasets, with the AE achieving particularly small angular errors in several cases. This establishes a comparable local geometric basis from which to examine whether the targeted global properties also emerge. See Appendix~\ref{app:models_training} for implementation details.

\paragraph{Do global-consistency properties emerge from reconstruction?}
We next assess the global structural properties of learned representations through the time-coherence and commutativity residuals reported in Table~\ref{table_toys}, examining whether accurate local tangent recovery also yields global consistency under a map-first objective. While our parametrization satisfies both properties by design, the AE can retain low tangent error despite substantial global-consistency violations. These results empirically support the counterexample of Appendix~\ref{app:ae_counterexample}, showing that accurate local geometric recovery does not imply a globally coherent structure.

\begin{table}[h!]
\centering
\begin{small}
\setlength{\tabcolsep}{3.5pt}
\begin{tabular}{lcccccc}
\toprule
& \multicolumn{2}{c}{\textbf{Manifold Tangent error} $\downarrow$}
& \multicolumn{2}{c}{\textbf{Time coherence error} $\downarrow$}
& \multicolumn{2}{c}{\textbf{Commutativity error} $\downarrow$} \\
\cmidrule(lr){2-3}
\cmidrule(lr){4-5}
\cmidrule(lr){6-7}
\textbf{Dataset}
& \textbf{Ours} & \textbf{AE}
& \textbf{Ours} & \textbf{AE}
& \textbf{Ours} & \textbf{AE} \\
\midrule

Paraboloid (2, 3)
& $0.08^\circ \pm 0.02$
& $0.0044^\circ \pm 0.008$
& $0$
& $0.35 \pm 0.79$
& $0$
& $0.0176 \pm 0.039$ \\

Sinusoid (2, 3)
& $0.027^\circ \pm 0.005$
& $0.0012^\circ \pm 0.0006$
& $0$
& $6.4\,10^{-9} \pm 5.8\,10^{-9}$
& $0$
& $4.7\,10^{-9} \pm 2.1\,10^{-9}$ \\

Swiss Roll (2, 3)
& $8.4^\circ \pm 3.3$
& $0.032^\circ \pm 0.0053$
& $0$
& $0.91 \pm 2.0$
& $0$
& $1431 \pm 3178$ \\

Sphere (2, 3)
& $0.89^\circ \pm 0.32$
& $0.0082^\circ \pm 0.004$
& $0$
& $0.35 \pm 0.73$
& $0$
& $0.075 \pm 0.16$ \\

\midrule

Helicoidal (4, 12)
& $0.86^\circ \pm 1.6$
& $0.149^\circ \pm 0.09$
& $0$
& $0.17 \pm 0.37$
& $0$
& $4.19 \pm 9.29$ \\

Paraboloid (8, 10)
& $0.36^\circ \pm 0.021$
& $0.0065^\circ \pm 0.0002$
& $0$
& $1.7\,10^{-6} \pm 2.6\,10^{-7}$
& $0$
& $8.2\,10^{-7} \pm 2\,10^{-7}$ \\

Paraboloid (20, 40)
& $0.37^\circ \pm 0.011$
& $0.077^\circ \pm 0.0019$
& $0$
& $3.0\,10^{-4} \pm 2.5\,10^{-5}$
& $0$
& $1.3\,10^{-4} \pm 9.3\,10^{-6}$ \\

Sinusoidal (8, 10)
& $1.8^\circ \pm 1.15$
& $0.21^\circ \pm 0.24$
& $0$
& $0.055 \pm 0.082$
& $0$
& $19.4 \pm 27.27$ \\

Sinusoidal (20, 40)
& $5.96^\circ \pm 0.61$
& $0.95^\circ \pm 0.19$
& $0$
& $0.041 \pm 0.027$
& $0$
& $1.55 \pm 1.18$ \\

\bottomrule
\end{tabular}
\end{small}

\caption{\textbf{Local geometry and global consistency.}
Tangent error measures the angular discrepancy (in degrees) between estimated and ground-truth tangent spaces, with smaller values indicating better local geometric recovery. Time and commutativity errors measure violations of the corresponding global-consistency properties defined in Appendix~\ref{app:global_consistency_metrics}, with zero indicating exact satisfaction. Our parametrization enforces both properties by construction, hence the zero residuals. The table therefore permits two complementary readings: tangent error evaluates local geometry, while time and commute errors evaluate the global organization of the learned directions. Across datasets, low tangent error does not systematically coincide with low global-consistency residuals for the AE, whereas our formulation preserves the targeted properties exactly.}
\label{table_toys}
\end{table}

\paragraph{Behavior beyond the global-coordinate assumption.}
We use the sphere \(\mathbb S^2\) to examine how violations of our global-coordinate assumption appear in the learned representation. Since \(\mathbb S^2\) does not admit a smooth global tangent frame, the targeted structure cannot hold everywhere; a further discussion of this topological obstruction is provided in Appendix~\ref{app:parallel}. As shown in Figure~\ref{fig::vector_field_sphere}, the learned fields remain tangent and locally organized over most of the surface, but develop degenerate regions near the poles. Consistently with the Hairy Ball Theorem, the topological obstruction thus appears explicitly in the learned vector fields rather than being hidden by a seemingly valid global representation.

\begin{figure}[h]
\centering
\begin{subfigure}{0.45\columnwidth}
\centering
\begin{tikzpicture}[scale=1.0]
\node[font=\small] at (0.4,1.3) {(a)};
\begin{axis}[
    width=\linewidth,
    height=0.9\linewidth,
    view={25}{25},
    axis lines=none,
    ticks=none,
    enlargelimits=0.1,
    clip=false,
]
\addplot3[
    quiver={
        u=\thisrow{u_all},
        v=\thisrow{v_all},
        w=\thisrow{w_all},
        scale arrows=0.3,
    },
    BlueCustom,
    line width=0.1pt,
    -{Stealth[length=1pt,width=1pt]},
] table[
    col sep=comma,
    x=X,
    y=Y,
    z=Z
] {vector_field_sphere_front.csv};
\end{axis}
\end{tikzpicture}
\end{subfigure}
\begin{subfigure}{0.45\columnwidth}
\centering
\begin{tikzpicture}[scale=1.0]
\node[font=\small] at (0.4,1.3) {(b)};
\begin{axis}[
    width=\linewidth,
    height=0.9\linewidth,
    view={25}{25},
    axis lines=none,
    ticks=none,
    enlargelimits=0.1,
    clip=false,
]
\addplot3[
    quiver={
        u=\thisrow{u_1},
        v=\thisrow{v_1},
        w=\thisrow{w_1},
        scale arrows=0.3,
    },
    BlueCustom,
    line width=0.1pt,
    -{Stealth[length=1pt,width=1pt]},
] table[
    col sep=comma,
    x=X,
    y=Y,
    z=Z
] {vector_field_sphere_front.csv};
\addplot3[
    quiver={
        u=\thisrow{u_2},
        v=\thisrow{v_2},
        w=\thisrow{w_2},
        scale arrows=0.3,
    },
    OrangeCustom,
    line width=0.1pt,
    -{Stealth[length=1pt,width=1pt]},
] table[
    col sep=comma,
    x=X,
    y=Y,
    z=Z
] {vector_field_sphere_front.csv};
\end{axis}
\end{tikzpicture}
\end{subfigure}
\caption{\textbf{Topological obstruction.}
\(\mathbb S^2\) does not admit a smooth global tangent frame.
(a) The combined transport field \(F\,T\) remains tangent over most of the
surface but develops a degenerate region near the pole.
(b) The two learned fields \(F_1\) and \(F_2\) provide locally consistent
tangent directions away from these regions, but cannot form a smooth
nowhere-degenerate global frame. This behavior is consistent with the
Hairy Ball Theorem and illustrates the limitation of global flow-based
coordinates on non-parallelizable manifolds.}
\label{fig::vector_field_sphere}
\end{figure}

\section{Limitations and Future Work}

The proposed framework targets a restricted geometric setting,
which leads to several limitations. First, enforcing non-contraction requires Jacobian-based quantities and constitutes the main computational cost of the implementation. This limitation concerns the enforcement of the geometric constraint rather than the property-driven formulation itself. Developing architectures that satisfy Lie-derivative or non-contraction constraints by construction could substantially improve the scalability of the method. 

Second, our theoretical framework assumes that the data manifold admits a global flow-based coordinate structure characterized in Definition~\ref{def::manifold}. This excludes manifolds with topological obstructions to a single global chart, as illustrated by the sphere experiment. A natural extension is to replace the single global reference by an atlas of locally consistent coordinate systems. In our formulation, this could be achieved by extending $u_s$ to multiple terminal points and combining the resulting local representations.

Finally, the method targets the geometry of the data manifold but does not model how probability mass is distributed over it. Consequently, learned coordinates characterize manifold structure rather than a generative density. Combining the proposed geometric representation with density estimation in learned coordinates would extend the framework toward likelihood-based evaluation, sampling, and generative modeling.

\section{Conclusion}
This work advocates a property-driven view of manifold representation learning. Our main observation is that global coordinate structure is not merely a consequence to inspect after learning a representation: it can itself serve as the object to be learned. In the flow-based setting considered, commutativity, time coherence, and a common reference characterize this structure, while non-contraction provides the geometric control needed when these properties are enforced approximately from samples. Our framework highlights that accurately recovering local tangent directions does not by itself yield a globally consistent coordinate system. Experiments show that these objectives can behave differently, and that enforcing global structure changes the learned representation rather than simply improving a local geometric metric. More broadly, our construction illustrates a general principle: representations can be learned from the geometric properties they should satisfy, rather than from an unconstrained map followed by a posteriori analysis. This perspective may support a basis for more interpretable and geometrically structured representations.

\bibliography{references}

\begin{thebibliography}{44}
\providecommand{\natexlab}[1]{#1}
\providecommand{\url}[1]{\texttt{#1}}
\expandafter\ifx\csname urlstyle\endcsname\relax
  \providecommand{\doi}[1]{doi: #1}\else
  \providecommand{\doi}{doi: \begingroup \urlstyle{rm}\Url}\fi

\bibitem[B{\'e}thune et~al.(2025)B{\'e}thune, Vigouroux, Du, VanRullen, Serre, and Boutin]{bethune2025follow}
Louis B{\'e}thune, David Vigouroux, Yilun Du, Rufin VanRullen, Thomas Serre, and Victor Boutin.
\newblock Follow the energy, find the path: Riemannian metrics from energy-based models.
\newblock In \emph{The Thirty-ninth Annual Conference on Neural Information Processing Systems}, 2025.
\newblock URL \url{https://openreview.net/forum?id=BOiQ7Kd5Lx}.

\bibitem[Chang et~al.(2019)Chang, Roohi, and Gao]{NEURIPS2019_2647c1db}
Ya-Chien Chang, Nima Roohi, and Sicun Gao.
\newblock Neural lyapunov control.
\newblock In H.~Wallach, H.~Larochelle, A.~Beygelzimer, F.~d\textquotesingle Alch\'{e}-Buc, E.~Fox, and R.~Garnett (eds.), \emph{Advances in Neural Information Processing Systems}, volume~32. Curran Associates, Inc., 2019.
\newblock URL \url{https://proceedings.neurips.cc/paper_files/paper/2019/file/2647c1dba23bc0e0f9cdf75339e120d2-Paper.pdf}.

\bibitem[Chen et~al.(2019)Chen, Behrmann, Duvenaud, and Jacobsen]{lipswish_act}
Ricky T.~Q. Chen, Jens Behrmann, David Duvenaud, and J\"{o}rn-Henrik Jacobsen.
\newblock \emph{Residual flows for invertible generative modeling}.
\newblock Curran Associates Inc., Red Hook, NY, USA, 2019.

\bibitem[Chen et~al.(2020)Chen, Kornblith, Norouzi, and Hinton]{chen2020simclr}
Ting Chen, Simon Kornblith, Mohammad Norouzi, and Geoffrey Hinton.
\newblock A simple framework for contrastive learning of visual representations.
\newblock In \emph{Proceedings of the 37th International Conference on Machine Learning}, ICML'20. JMLR.org, 2020.

\bibitem[Courts \& Kvinge(2022)Courts and Kvinge]{courts2022bundle}
Nico Courts and Henry Kvinge.
\newblock Bundle networks: Fiber bundles, local trivializations, and a generative approach to exploring many-to-one maps.
\newblock In \emph{International Conference on Learning Representations}, 2022.
\newblock URL \url{https://openreview.net/forum?id=aBXzcPPOuX}.

\bibitem[de~Kruiff et~al.(2025)de~Kruiff, Bekkers, {\"O}ktem, Sch{\"o}nlieb, and Diepeveen]{kruiff2025pullback}
Friso de~Kruiff, Erik~J Bekkers, Ozan {\"O}ktem, Carola-Bibiane Sch{\"o}nlieb, and Willem Diepeveen.
\newblock Pullback flow matching on data manifolds, 2025.
\newblock URL \url{https://openreview.net/forum?id=mBXLtNKpeQ}.

\bibitem[Diepeveen \& Needell(2025)Diepeveen and Needell]{diepeveen2025manifoldlearningnormalizingflows}
Willem Diepeveen and Deanna Needell.
\newblock Manifold learning with normalizing flows: Towards regularity, expressivity and iso-riemannian geometry, 2025.
\newblock URL \url{https://arxiv.org/abs/2505.08087}.

\bibitem[Diepeveen et~al.(2025)Diepeveen, Batzolis, Shumaylov, and Sch\"{o}nlieb]{pmlr-v267-diepeveen25a}
Willem Diepeveen, Georgios Batzolis, Zakhar Shumaylov, and Carola-Bibiane Sch\"{o}nlieb.
\newblock Score-based pullback {R}iemannian geometry: Extracting the data manifold geometry using anisotropic flows.
\newblock In Aarti Singh, Maryam Fazel, Daniel Hsu, Simon Lacoste-Julien, Felix Berkenkamp, Tegan Maharaj, Kiri Wagstaff, and Jerry Zhu (eds.), \emph{Proceedings of the 42nd International Conference on Machine Learning}, volume 267 of \emph{Proceedings of Machine Learning Research}, pp.\  13746--13773. PMLR, 13--19 Jul 2025.
\newblock URL \url{https://proceedings.mlr.press/v267/diepeveen25a.html}.

\bibitem[Donahue et~al.(2017)Donahue, Kr{\"a}henb{\"u}hl, and Darrell]{donahue2017adversarial}
Jeff Donahue, Philipp Kr{\"a}henb{\"u}hl, and Trevor Darrell.
\newblock Adversarial feature learning.
\newblock In \emph{International Conference on Learning Representations}, 2017.
\newblock URL \url{https://openreview.net/forum?id=BJtNZAFgg}.

\bibitem[Dumoulin et~al.(2017)Dumoulin, Belghazi, Poole, Lamb, Arjovsky, Mastropietro, and Courville]{dumoulin2017adversarially}
Vincent Dumoulin, Ishmael Belghazi, Ben Poole, Alex Lamb, Martin Arjovsky, Olivier Mastropietro, and Aaron Courville.
\newblock Adversarially learned inference.
\newblock In \emph{International Conference on Learning Representations}, 2017.
\newblock URL \url{https://openreview.net/forum?id=B1ElR4cgg}.

\bibitem[Durkan et~al.(2019)Durkan, Bekasov, Murray, and Papamakarios]{NEURIPS2019_7ac71d43}
Conor Durkan, Artur Bekasov, Iain Murray, and George Papamakarios.
\newblock Neural spline flows.
\newblock In H.~Wallach, H.~Larochelle, A.~Beygelzimer, F.~d\textquotesingle Alch\'{e}-Buc, E.~Fox, and R.~Garnett (eds.), \emph{Advances in Neural Information Processing Systems}, volume~32. Curran Associates, Inc., 2019.
\newblock URL \url{https://proceedings.neurips.cc/paper_files/paper/2019/file/7ac71d433f282034e088473244df8c02-Paper.pdf}.

\bibitem[Frans et~al.(2025)Frans, Levine, and Abbeel]{frans2025a}
Kevin Frans, Sergey Levine, and Pieter Abbeel.
\newblock A stable whitening optimizer for efficient neural network training.
\newblock In \emph{The Thirty-ninth Annual Conference on Neural Information Processing Systems}, 2025.
\newblock URL \url{https://openreview.net/forum?id=0T8i3uXq3O}.

\bibitem[Guillemin \& Pollack(1974)Guillemin and Pollack]{GuilleminPollack}
Victor Guillemin and Alan Pollack.
\newblock \emph{Differential topology}.
\newblock Prentice-Hall, Ine., Englewood Cliffs, New Jersey, 1974.

\bibitem[Hahm et~al.(2024)Hahm, Lee, Kim, and Lee]{hahm2024isometric}
Jaehoon Hahm, Junho Lee, Sunghyun Kim, and Joonseok Lee.
\newblock Isometric representation learning for disentangled latent space of diffusion models, 2024.
\newblock URL \url{https://openreview.net/forum?id=QIrYb3Vlze}.

\bibitem[He et~al.(2020)He, Fan, Wu, Xie, and Girshick]{he2020moco}
Kaiming He, Haoqi Fan, Yuxin Wu, Saining Xie, and Ross Girshick.
\newblock Momentum contrast for unsupervised visual representation learning.
\newblock In \emph{Proceedings of the IEEE/CVF Conference on Computer Vision and Pattern Recognition (CVPR)}, June 2020.

\bibitem[Higgins et~al.(2017)Higgins, Matthey, Pal, Burgess, Glorot, Botvinick, Mohamed, and Lerchner]{higgins2017betavae}
Irina Higgins, Loic Matthey, Arka Pal, Christopher Burgess, Xavier Glorot, Matthew Botvinick, Shakir Mohamed, and Alexander Lerchner.
\newblock beta-{VAE}: Learning basic visual concepts with a constrained variational framework.
\newblock In \emph{International Conference on Learning Representations}, 2017.
\newblock URL \url{https://openreview.net/forum?id=Sy2fzU9gl}.

\bibitem[Higgins et~al.(2018)Higgins, Amos, Pfau, Racani{\`{e}}re, Matthey, Rezende, and Lerchner]{higgins2018towards}
Irina Higgins, David Amos, David Pfau, S{\'{e}}bastien Racani{\`{e}}re, Lo{\"{\i}}c Matthey, Danilo~J. Rezende, and Alexander Lerchner.
\newblock Towards a definition of disentangled representations.
\newblock \emph{CoRR}, abs/1812.02230, 2018.
\newblock URL \url{http://arxiv.org/abs/1812.02230}.

\bibitem[Jia et~al.(2019)Jia, Zaharia, and Aiken]{soap}
Zhihao Jia, Matei Zaharia, and Alex Aiken.
\newblock Beyond data and model parallelism for deep neural networks.
\newblock In A.~Talwalkar, V.~Smith, and M.~Zaharia (eds.), \emph{Proceedings of Machine Learning and Systems}, volume~1, pp.\  1--13, 2019.
\newblock URL \url{https://proceedings.mlsys.org/paper_files/paper/2019/file/b422680f3db0986ddd7f8f126baaf0fa-Paper.pdf}.

\bibitem[Kaba et~al.(2023)Kaba, Mondal, Zhang, Bengio, and Ravanbakhsh]{pmlr-v202-kaba23a}
S\'{e}kou-Oumar Kaba, Arnab~Kumar Mondal, Yan Zhang, Yoshua Bengio, and Siamak Ravanbakhsh.
\newblock Equivariance with learned canonicalization functions.
\newblock In Andreas Krause, Emma Brunskill, Kyunghyun Cho, Barbara Engelhardt, Sivan Sabato, and Jonathan Scarlett (eds.), \emph{Proceedings of the 40th International Conference on Machine Learning}, volume 202 of \emph{Proceedings of Machine Learning Research}, pp.\  15546--15566. PMLR, 23--29 Jul 2023.
\newblock URL \url{https://proceedings.mlr.press/v202/kaba23a.html}.

\bibitem[Kapu\'{n}niak et~al.(2024)Kapu\'{n}niak, Potaptchik, Reu, Zhang, Tong, Bronstein, Bose, and Di~Giovanni]{metric_flow_matching}
Kacper Kapu\'{n}niak, Peter Potaptchik, Teodora Reu, Leo Zhang, Alexander Tong, Michael Bronstein, Avishek~Joey Bose, and Francesco Di~Giovanni.
\newblock Metric flow matching for smooth interpolations on the data manifold.
\newblock In \emph{Proceedings of the 38th International Conference on Neural Information Processing Systems}, NIPS '24, Red Hook, NY, USA, 2024. Curran Associates Inc.
\newblock ISBN 9798331314385.

\bibitem[Kim et~al.(2018)Kim, Gilmer, Wattenberg, and Viégas]{kim2018tcav}
Been Kim, Justin Gilmer, Martin Wattenberg, and Fernanda Viégas.
\newblock {TCAV}: Relative concept importance testing with linear concept activation vectors, 2018.
\newblock URL \url{https://openreview.net/forum?id=S1viikbCW}.

\bibitem[Kim \& Mnih(2018)Kim and Mnih]{pmlr-v80-kim18b}
Hyunjik Kim and Andriy Mnih.
\newblock Disentangling by factorising.
\newblock In Jennifer Dy and Andreas Krause (eds.), \emph{Proceedings of the 35th International Conference on Machine Learning}, volume~80 of \emph{Proceedings of Machine Learning Research}, pp.\  2649--2658. PMLR, 10--15 Jul 2018.
\newblock URL \url{https://proceedings.mlr.press/v80/kim18b.html}.

\bibitem[Kingma \& Welling(2014)Kingma and Welling]{kingma2014vae}
Diederik~P. Kingma and Max Welling.
\newblock Auto-encoding variational bayes.
\newblock In Yoshua Bengio and Yann LeCun (eds.), \emph{2nd International Conference on Learning Representations, {ICLR} 2014, Banff, AB, Canada, April 14-16, 2014, Conference Track Proceedings}, 2014.
\newblock URL \url{http://arxiv.org/abs/1312.6114}.

\bibitem[Koh et~al.(2020)Koh, Nguyen, Tang, Mussmann, Pierson, Kim, and Liang]{pmlr-v119-koh20a}
Pang~Wei Koh, Thao Nguyen, Yew~Siang Tang, Stephen Mussmann, Emma Pierson, Been Kim, and Percy Liang.
\newblock Concept bottleneck models.
\newblock In Hal~Daumé III and Aarti Singh (eds.), \emph{Proceedings of the 37th International Conference on Machine Learning}, volume 119 of \emph{Proceedings of Machine Learning Research}, pp.\  5338--5348. PMLR, 13--18 Jul 2020.
\newblock URL \url{https://proceedings.mlr.press/v119/koh20a.html}.

\bibitem[Lazarev(2025)]{lazarev2025ricci}
Alexey Lazarev.
\newblock \emph{{Metric regularization in machine learning with curvature functionals}}.
\newblock Theses, {Universit{\'e} de Toulouse}, May 2025.
\newblock URL \url{https://theses.hal.science/tel-05223442}.

\bibitem[LeCun et~al.(2006)LeCun, Chopra, Hadsell, Ranzato, and Huang]{LeCun2006ATO}
Yann LeCun, Sumit Chopra, Raia Hadsell, Aurelio Ranzato, and Fu~Jie Huang.
\newblock A tutorial on energy-based learning.
\newblock In \emph{Predicting Structured Data}, 2006.
\newblock URL \url{http://yann.lecun.com/exdb/publis/pdf/lecun-06.pdf}.

\bibitem[Lee(2013)]{lee_introduction_2012}
John~M. Lee.
\newblock \emph{Introduction to smooth manifolds}.
\newblock Graduate texts in mathematics. Springer, New York, 2nd edition edition, 2013.

\bibitem[Lee(2018)]{leeIntroductionRiemannianManifolds2018}
John~M. Lee.
\newblock \emph{Introduction to {Riemannian} {Manifolds}}, volume 176 of \emph{Graduate {Texts} in {Mathematics}}.
\newblock Springer International Publishing, Cham, 2018.
\newblock ISBN 978-3-319-91754-2 978-3-319-91755-9.
\newblock \doi{10.1007/978-3-319-91755-9}.
\newblock URL \url{http://link.springer.com/10.1007/978-3-319-91755-9}.

\bibitem[Lei \& Daniilidis(2022)Lei and Daniilidis]{Lei_2022_CVPR}
Jiahui Lei and Kostas Daniilidis.
\newblock Cadex: Learning canonical deformation coordinate space for dynamic surface representation via neural homeomorphism.
\newblock In \emph{Proceedings of the IEEE/CVF Conference on Computer Vision and Pattern Recognition (CVPR)}, pp.\  6624--6634, June 2022.

\bibitem[Lipman et~al.(2023)Lipman, Chen, Ben-Hamu, Nickel, and Le]{lipman2023flow}
Yaron Lipman, Ricky T.~Q. Chen, Heli Ben-Hamu, Maximilian Nickel, and Matthew Le.
\newblock Flow matching for generative modeling.
\newblock In \emph{The Eleventh International Conference on Learning Representations}, 2023.
\newblock URL \url{https://openreview.net/forum?id=PqvMRDCJT9t}.

\bibitem[Locatello et~al.(2019)Locatello, Bauer, Lucic, Raetsch, Gelly, Sch{\"o}lkopf, and Bachem]{pmlr-v97-locatello19a}
Francesco Locatello, Stefan Bauer, Mario Lucic, Gunnar Raetsch, Sylvain Gelly, Bernhard Sch{\"o}lkopf, and Olivier Bachem.
\newblock Challenging common assumptions in the unsupervised learning of disentangled representations.
\newblock In Kamalika Chaudhuri and Ruslan Salakhutdinov (eds.), \emph{Proceedings of the 36th International Conference on Machine Learning}, volume~97 of \emph{Proceedings of Machine Learning Research}, pp.\  4114--4124. PMLR, 09--15 Jun 2019.
\newblock URL \url{https://proceedings.mlr.press/v97/locatello19a.html}.

\bibitem[Meilă \& Zhang(2023)Meilă and Zhang]{manifoldlearningwhathow}
Marina Meilă and Hanyu Zhang.
\newblock Manifold learning: what, how, and why, 2023.
\newblock URL \url{https://arxiv.org/abs/2311.03757}.

\bibitem[Rezende et~al.(2014)Rezende, Mohamed, and Wierstra]{pmlr-v32-rezende14}
Danilo~Jimenez Rezende, Shakir Mohamed, and Daan Wierstra.
\newblock Stochastic backpropagation and approximate inference in deep generative models.
\newblock In Eric~P. Xing and Tony Jebara (eds.), \emph{Proceedings of the 31st International Conference on Machine Learning}, volume~32 of \emph{Proceedings of Machine Learning Research}, pp.\  1278--1286, Bejing, China, 22--24 Jun 2014. PMLR.
\newblock URL \url{https://proceedings.mlr.press/v32/rezende14.html}.

\bibitem[Rodriguez et~al.(2022)Rodriguez, Ames, and Yue]{pmlr-v162-rodriguez22a}
Ivan Dario~Jimenez Rodriguez, Aaron Ames, and Yisong Yue.
\newblock {L}ya{N}et: A {L}yapunov framework for training neural {ODE}s.
\newblock In Kamalika Chaudhuri, Stefanie Jegelka, Le~Song, Csaba Szepesvari, Gang Niu, and Sivan Sabato (eds.), \emph{Proceedings of the 39th International Conference on Machine Learning}, volume 162 of \emph{Proceedings of Machine Learning Research}, pp.\  18687--18703. PMLR, 17--23 Jul 2022.
\newblock URL \url{https://proceedings.mlr.press/v162/rodriguez22a.html}.

\bibitem[Shao et~al.(2018)Shao, Kumar, and Fletcher]{Shao2017TheRG}
Hang Shao, Abhishek Kumar, and P.~Thomas Fletcher.
\newblock The riemannian geometry of deep generative models.
\newblock \emph{2018 IEEE/CVF Conference on Computer Vision and Pattern Recognition Workshops (CVPRW)}, pp.\  428--4288, 2018.
\newblock URL \url{https://openaccess.thecvf.com/content_cvpr_2018_workshops/papers/w10/Shao_The_Riemannian_Geometry_CVPR_2018_paper.pdf}.

\bibitem[Shen et~al.(2020)Shen, Gu, Tang, and Zhou]{Shen_2020_CVPR}
Yujun Shen, Jinjin Gu, Xiaoou Tang, and Bolei Zhou.
\newblock Interpreting the latent space of gans for semantic face editing.
\newblock In \emph{Proceedings of the IEEE/CVF Conference on Computer Vision and Pattern Recognition (CVPR)}, June 2020.

\bibitem[Song et~al.(2020)Song, Sohl-Dickstein, Kingma, Kumar, Ermon, and Poole]{Song2020ScoreBasedGM}
Yang Song, Jascha~Narain Sohl-Dickstein, Diederik~P. Kingma, Abhishek Kumar, Stefano Ermon, and Ben Poole.
\newblock Score-based generative modeling through stochastic differential equations.
\newblock \emph{ArXiv}, abs/2011.13456, 2020.
\newblock URL \url{https://openreview.net/pdf/ef0eadbe07115b0853e964f17aa09d811cd490f1.pdf?ref=news-tutorials-ai-research}.

\bibitem[Takehiro~Aoshima(2023)]{Aoshima_2023_CVPR}
Takashi~Matsubara Takehiro~Aoshima.
\newblock Deep curvilinear editing: Commutative and nonlinear image manipulation for pretrained deep generative model.
\newblock In \emph{Proceedings of the IEEE/CVF Conference on Computer Vision and Pattern Recognition (CVPR)}, 2023.

\bibitem[Teschl(2012)]{Teschl}
Gerald Teschl.
\newblock \emph{Ordinary Differential Equations and Dynamical Systems}, volume 140 of \emph{Graduate Studies in Mathematics}.
\newblock American Mathematical Society, 2012.
\newblock ISBN 978-0-8218-8328-0.
\newblock Graduate Studies in Mathematics 144, Amer. Math. Soc., Providence, 2012.

\bibitem[Thomas et~al.(2023)Thomas, Agustin, Louis, Thibaut, David, Julien, Rémi, and Thomas]{fel2023craft}
Fel Thomas, Picard Agustin, Bethune Louis, Boissin Thibaut, Vigouroux David, Colin Julien, Cadène Rémi, and Serre Thomas.
\newblock Craft: Concept recursive activation factorization for explainability.
\newblock In \emph{Proceedings of the IEEE Conference on Computer Vision and Pattern Recognition (CVPR)}, 2023.

\bibitem[woo Lee et~al.(2025)woo Lee, Choi, and Kim]{lee2025robust}
Hyun woo Lee, Hayoung Choi, and Hyunju Kim.
\newblock Robust weight initialization for tanh neural networks with fixed point analysis.
\newblock In \emph{The Thirteenth International Conference on Learning Representations}, 2025.
\newblock URL \url{https://openreview.net/forum?id=Es4RPNDtmq}.

\bibitem[Zbontar et~al.(2021)Zbontar, Jing, Misra, LeCun, and Deny]{zbontar2021barlow}
Jure Zbontar, Li~Jing, Ishan Misra, Yann LeCun, and Stephane Deny.
\newblock Barlow twins: Self-supervised learning via redundancy reduction.
\newblock In Marina Meila and Tong Zhang (eds.), \emph{Proceedings of the 38th International Conference on Machine Learning}, volume 139 of \emph{Proceedings of Machine Learning Research}, pp.\  12310--12320. PMLR, 18--24 Jul 2021.
\newblock URL \url{https://proceedings.mlr.press/v139/zbontar21a.html}.

\bibitem[Zeng et~al.(2024)Zeng, De~Jesús, Fox, and Graham]{Zeng_2024}
Kevin Zeng, Carlos E~Pérez De~Jesús, Andrew~J Fox, and Michael~D Graham.
\newblock Autoencoders for discovering manifold dimension and coordinates in data from complex dynamical systems.
\newblock \emph{Machine Learning: Science and Technology}, 5\penalty0 (2):\penalty0 025053, may 2024.
\newblock \doi{10.1088/2632-2153/ad4ba5}.
\newblock URL \url{https://doi.org/10.1088/2632-2153/ad4ba5}.

\bibitem[Zhu et~al.(2021)Zhu, Xu, and Tao]{zhu2021commutative}
Xinqi Zhu, Chang Xu, and Dacheng Tao.
\newblock Commutative lie group vae for disentanglement learning.
\newblock In Marina Meila and Tong Zhang (eds.), \emph{Proceedings of the 38th International Conference on Machine Learning}, volume 139 of \emph{Proceedings of Machine Learning Research}, pp.\  12924--12934. PMLR, 18--24 Jul 2021.
\newblock URL \url{https://proceedings.mlr.press/v139/zhu21f.html}.

\end{thebibliography}
\bibliographystyle{tmlr}

\newpage

\appendix
\onecolumn

\section{Appendix - Background}
\label{app:background}

Before introducing the Lie derivative formally, we provide an intuitive
interpretation of the quantity we seek to control. A vector field defines a
continuous deformation of space by moving points along its flow. As points are
transported, the distances between nearby points may increase, remain unchanged,
or decrease. In many learning settings, unconstrained vector fields tend to
collapse regions of space, shrinking distances and destroying geometric
structure. To prevent such degeneracies, we need a way to measure how a vector
field locally changes distances between infinitesimally close points. The Lie
derivative of the metric provides exactly this information: it quantifies the
instantaneous rate at which the squared distance between nearby points changes
when they are pushed along the flow of the vector field. Requiring this quantity
to be non-negative enforces a local non-contraction property, ensuring that the
flow does not shrink distances to first order.

\subsection{Lie Bracket of Vector Fields}

Let $M$ be a smooth manifold and let $X, Y \in \mathfrak{X}(M)$ be smooth vector fields.
Interpreting vector fields as first–order differential operators, their \emph{Lie bracket} is given by the commutator
\[
[X,Y] = XY - YX.
\]

In local coordinates, writing the vector fields as
\[
X = \sum_{i} X^i\,\partial_{x^i}, 
\qquad
Y = \sum_{i} Y^i\,\partial_{x^i},
\]
the Lie bracket can be expressed in matrix form as
\[
[X,Y]
= (J_Y)\,X - (J_X)\,Y,
\]
where $J_X$ and $J_Y$ denote the Jacobian matrices of the component functions of $X$ and $Y$, respectively.

\subsection{Pushforward and the Pushforward Identity}

Let $\varphi : M \to N$ be a smooth map between manifolds.  
The \emph{pushforward} of a tangent vector $v \in T_x M$ is the linear map
\[
\varphi_* : T_x M \to T_{\varphi(x)} N,
\qquad 
\varphi_*(v)(f) = v(f \circ \varphi)
\]
for all $f \in C^\infty(N)$.

If $\{\varphi_t\}$ is the flow of a vector field $X$, then the infinitesimal change of a vector field $Y$ under transport by this flow is given by the \emph{pushforward identity}:
\[
\frac{d}{dt}\bigg|_{t=0} (\varphi_t)_* Y = [X, Y].
\]

\subsection{Lie Derivative of the Metric (Euclidean Case)}

In Euclidean space $\mathbb{R}^n$ equipped with the standard metric 
\[
g_x(u,v) \equiv g(u,v) = u \cdot v,
\]
the metric is constant and coincides with the dot product.
For a smooth vector field $X : \mathbb{R}^n \to \mathbb{R}^n$, the Lie derivative of the metric with respect to $X$ is the symmetric $(0,2)$-tensor
\[
(\mathcal{L}_X g)(Y, Z)
= X(Y \cdot Z) - [X,Y]\cdot Z - Y\cdot [X,Z].
\]

In coordinates, writing
\[
X = (X^1,\dots,X^n), \qquad Y = (Y^1,\dots,Y^n), \qquad Z = (Z^1,\dots,Z^n),
\]
and denoting by $\partial_i$ the partial derivative with respect to $x^i$, this expression reduces to
\[
(\mathcal{L}_X g)_{ij}
= \partial_i X_j + \partial_j X_i.
\]

Thus, in $\mathbb{R}^n$ the Lie derivative of the metric is simply the symmetric part of the Jacobian matrix of $X$.  
A vector field satisfies $\mathcal{L}_X g = 0$ if and only if
\[
\partial_i X_j + \partial_j X_i = 0,
\]
which characterizes infinitesimal rigid motions (translations and rotations). \\

For a conformal Euclidean metric of the form \(g(x) = e^{\sigma(x)} I_n\), the Lie derivative of the metric along a vector field \(X\) is given by
\[
(\mathcal{L}_X g)_{ij}
= \langle \nabla \sigma, X \rangle\, \delta_{ij}
+ \partial_i X_j
+ \partial_j X_i .
\]

Let $\Phi_X^t$ denote the flow generated by $X$. For any $x \in \mathbb{R}^n$ and
any $v \in T_x \mathbb{R}^n \simeq \mathbb{R}^n$,
\[
\left.\frac{d}{dt}\right|_{t=0}
\big\|\mathrm{d}\Phi_X^t(x)\, v\big\|^2
=
\left.\frac{d}{dt}\right|_{t=0}
g\!\big(\mathrm{d}\Phi_X^t(x)\, v,\ \mathrm{d}\Phi_X^t(x)\, v\big)
=
(\mathcal{L}_X g)_x(v,v).
\]
Hence, if $\mathcal{L}_X g \succeq 0$, then for all $x$ and all $v$,
\[
\left.\frac{d}{dt}\right|_{t=0}
\big\|\mathrm{d}\Phi_X^t(x)\, v\big\|^2 \ge 0,
\]
i.e., the squared norm of any infinitesimal displacement does not decrease to
first order under the flow, which corresponds to local non-contraction of
distances.

\subsection{Commuting Vector Fields}

Two vector fields $X, Y \in \mathfrak{X}(M)$ are said to \emph{commute} if their Lie bracket vanishes:
\[
[X, Y] = 0.
\]
This condition is equivalent to any of the following:

\begin{itemize}
    \item the flows of $X$ and $Y$ commute locally;
    \item there exist local coordinates in which both appear as coordinate vector fields;
    \item for all $f \in C^\infty(M)$, $X(Yf) = Y(Xf)$.
\end{itemize}

\subsection{Arc-length}
For each flow $\phi_i$ defined by $(F_i,T_i)$, let $z_i:[0,T_i(x)]\to\mathbb{R}^n$ denote the corresponding trajectory,
\[
\phi_i(x)=z_i\!\left(T_i(x),x\right), \qquad
\begin{cases}
\displaystyle \frac{d}{dt}\, z_i(t,x) = F_i\big(z_i(t,x)\big),\\
z_i(0,x)=\phi_{i-1} \circ ... \circ \phi_1(x).
\end{cases}
\]
We define the arc-length with the metric $G_i$ traveled along the $i$-th flow by
\[
\ell_i(x) \;=\; \int_{0}^{T_i(x)} \sqrt{\dot z_i(t)^\top G_i\!\left(z_i(t)\right)\dot z_i(t)}\,dt
\;=\; \int_{0}^{T_i(x)} \sqrt{F_i\!\left(z_i(t)\right)^\top G_i\!\left(z_i(t)\right)F_i\!\left(z_i(t)\right)}\,dt.
\]
The arc-lengths traveled along each flow, $(\ell_1(x),\dots,\ell_m(x))$, serve as the coordinates. 

\section{Appendix - Small reconstruction error does not imply coordinate consistency}
\label{app:ae_counterexample}

\begin{proposition}[Small reconstruction error does not imply coordinate consistency]
\label{prop:ae_counterexample}
Let
\[
\mathcal M=[-1,1]^2,
\qquad
E_\varepsilon(x)=x,
\]
and, for $\varepsilon>0$, define
\[
D_\varepsilon(z)
=
z+
\frac{\varepsilon}{\sqrt{2}}
\begin{bmatrix}
\tanh(z_2/\varepsilon)\\
\tanh(z_1/\varepsilon)
\end{bmatrix}.
\]
Then
\[
\sup_{x\in\mathcal M}
\|D_\varepsilon(E_\varepsilon(x))-x\|_2
\leq \varepsilon.
\]

Define
\[
T(x)=E_\varepsilon(x)=x,
\qquad
F_i(x)
=
-\partial_{z_i}D_\varepsilon(E_\varepsilon(x)),
\quad i=1,2.
\]
Then the time-coherence error does not vanish as
$\varepsilon\to0$:
\[
\sup_{x\in\mathcal M}
\|J_T(x)F(x)+I_2\|_F
\geq 1,
\]
while the commutativity error satisfies
\[
\sup_{x\in\mathcal M}
\|[F_1,F_2](x)\|_2
\geq
\frac{\sqrt{2}\,\sech^4(1)\tanh(1)}{\varepsilon}.
\]
Hence arbitrarily accurate reconstruction does not imply approximate
time coherence or commutativity of the induced coordinate directions.
\end{proposition}

\begin{proof}
Since $|\tanh(u)|\leq1$,
\[
\|D_\varepsilon(E_\varepsilon(x))-x\|_2
\leq
\frac{\varepsilon}{\sqrt{2}}\sqrt{2}
=
\varepsilon.
\]

Moreover,
\[
F_1(x)
=
\begin{bmatrix}
-1\\
-\frac{1}{\sqrt{2}}\sech^2(x_1/\varepsilon)
\end{bmatrix},
\qquad
F_2(x)
=
\begin{bmatrix}
-\frac{1}{\sqrt{2}}\sech^2(x_2/\varepsilon)\\
-1
\end{bmatrix}.
\]
Since $J_T=I_2$, at $x=0$,
\[
\|J_T(0)F(0)+I_2\|_F=1.
\]

Let
\[
p_\varepsilon(u)
=
\frac{1}{\sqrt{2}}\sech^2(u/\varepsilon).
\]
Then
\[
F_1=(-1,-p_\varepsilon(x_1))^\top,
\qquad
F_2=(-p_\varepsilon(x_2),-1)^\top,
\]
and
\[
[F_1,F_2]
=
\begin{bmatrix}
p_\varepsilon(x_1)p_\varepsilon'(x_2)\\
-p_\varepsilon'(x_1)p_\varepsilon(x_2)
\end{bmatrix},
\]
where
\[
p_\varepsilon'(u)
=
-\frac{\sqrt{2}}{\varepsilon}
\sech^2(u/\varepsilon)\tanh(u/\varepsilon).
\]
Evaluating at $x_\varepsilon=(\varepsilon,\varepsilon)$ gives
\[
\|[F_1,F_2](x_\varepsilon)\|_2
=
\frac{\sqrt{2}\,\sech^4(1)\tanh(1)}{\varepsilon},
\]
which diverges as $\varepsilon\to0$.
\end{proof}

\section{Discussion on parallelizability}
\label{app:parallel}

Parallelizability is a concept in differential geometry that characterizes the global structure of a manifold’s tangent bundle. A smooth \(m\)-dimensional manifold \(M\) is said to be \emph{parallelizable} if its tangent bundle \(TM\) is trivial, meaning that there exist \(m\) smooth vector fields that are everywhere linearly independent and form a global frame. Equivalently, parallelizability ensures the existence of a smoothly varying basis of the tangent space at every point of the manifold \citep[Cor.~10.20, p.~259 and Example 10.10c]{lee_introduction_2012}.

Importantly, our approach does \emph{not} assume strict global parallelizability as a topological property of the underlying data manifold.
Instead, it is motivated by the geometric intuition provided by parallelizable manifolds and relies on the existence of smooth vector fields that span the tangent space on a large subset of the manifold.
This distinction is crucial, as many manifolds of practical interest are not globally parallelizable, yet still admit well-defined tangent directions almost everywhere.

Results in differential topology show that for any smooth \(n\)-dimensional manifold, one can construct collections of smooth vector fields that are linearly independent outside a closed subset of codimension at least two. In particular, this exceptional set has measure zero with respect to any smooth volume form. As a consequence, the tangent bundle admits a smooth framing on an open dense subset of the manifold. These results follow from parametric transversality arguments applied to sections of the tangent bundle—specifically, from the transversality theorem ensuring that generic vector fields are transverse to the zero section \citep[Transversality Theorem, p68]{GuilleminPollack}.
This viewpoint is consistent with the local triviality of vector bundles and the theory of smooth frames described in modern treatments of smooth manifolds \citep[Corollary 10.20, p259]{lee_introduction_2012}.

From the perspective of our method, this means that a global framing need not exist everywhere for the construction to be meaningful.
Learning vector fields that form a frame almost everywhere is sufficient to recover intrinsic directions and define a coherent coordinate system on the data manifold, up to a negligible singular set.
This relaxed viewpoint explains why the method remains applicable even when strict global parallelizability fails.

\section{Flow-time and arc-length coordinates}
\label{app:time_length_coordinates}

Let \(T=(T_1,\ldots,T_m)\) and \(F=(F_1,\ldots,F_m)\) define global flow-based coordinates and satisfy

$$
\nabla T_j^\top F_i=-\delta_{ij}.
$$

Along an integral curve \(\gamma_i(t)=\Phi_i^t(x)\),

$$
\frac{d}{dt}T_j(\gamma_i(t))
=
\nabla T_j(\gamma_i(t))^\top F_i(\gamma_i(t))
=
-\delta_{ij},
$$

and therefore

$$
T_i(\Phi_i^t(x))=T_i(x)-t,
\qquad
T_j(\Phi_i^t(x))=T_j(x),\quad j\neq i.
$$

Hence \(T_i(x)\) is the flow time required to reach the section \(T_i=0\), while leaving the remaining coordinates unchanged.

Let \(x_i=(\phi_{i-1}\circ\cdots\circ\phi_1)(x)\) be the starting point of the \(i\)-th flow segment. Its signed intrinsic length is

$$
\ell_i(x)
=
\int_0^{T_i(x)}
\left\|
F_i\!\left(\Phi_i^t(x_i)\right)
\right\|_g\,dt .
$$

Define the associated arc-length map

$$
s_i(t)
=
\int_0^t
\left\|
F_i\!\left(\Phi_i^\tau(x_i)\right)
\right\|_g\,d\tau .
$$

Since \(F_i\) is nowhere vanishing,

$$
s_i'(t)
=
\left\|F_i(\Phi_i^t(x_i))\right\|_g
>0.
$$

Thus \(s_i\) is a smooth strictly monotone reparametrization of the flow parameter and is a diffeomorphism onto its image~\citep[Proposition~2.49]{leeIntroductionRiemannianManifolds2018}. Consequently,

$$
\ell_i(x)=s_i(T_i(x)),
\qquad
T_i(x)=s_i^{-1}(\ell_i(x)).
$$

Therefore the time and length coordinates describe the same integral curves with different parametrizations. In particular, if \(F_i\) is unit speed with respect to \(g\), then

$$
\ell_i(x)=T_i(x).
$$

To construct a global coordinate system, we use the learned time functions as a proxy. This proxy is exact when the vector field has unit norm, in which case the flow time to the reference point coincides with the traveled distance. In the general case, the exact quantity is \(T_i^*(x)=T_i(x)\|F_i(x)\|_2\). Since the trajectory is determined by the pair \((F_i,T_i)\), the same dynamics can be equivalently reparameterized as 
\[
F_i^*(x)=\frac{F_i(x)}{\|F_i(x)\|_2}, \quad T_i^*(x)=T_i(x)\|F_i(x)\|_2.
\]
Hence, using \(T_i(x)\) amounts to working with a normalized representation of the same trajectory.

\newpage

\section{Appendix - Proofs of Theorems and Propositions}
\label{app:proof}

\subsection{Proof of Theorem~\ref{thm:global_coordinates_coherence}}
\label{app:characterization_global_flow_coordinates}

\begin{proof}
Assume first that \(T\) and \(F\) define global flow-based coordinates.
Then, in the coordinates \(T=(T_1,\ldots,T_m)\), the fields \(F_i\)
correspond to the negative coordinate directions. Since coordinate vector
fields commute~\citep[Eq.~(8.10)]{lee_introduction_2012},
\[
[F_i,F_j]=0.
\]
Moreover,
\[
dT_j(F_i)=-\delta_{ij},
\]
which gives the time-coherence condition. Along the flow of \(F_i\),
\[
\frac{d}{dt}T_j(\Phi_i^t(x))
=
dT_j(F_i)
=
-\delta_{ij},
\]
and therefore
\[
T(\Phi_i^t(x))=T(x)-t e_i.
\]
Successively applying the flows for times
\(T_1(x),\ldots,T_m(x)\) thus sends every coordinate to zero, yielding the
common reference \(C=T^{-1}(0)\). The flows are independent of their order
by the equivalence between commuting vector fields and commuting flows
~\citep[Theorem~9.44]{lee_introduction_2012}.

Conversely, assume commutativity, time coherence, and the common-reference
condition. Time coherence gives
\[
dT_x(F_i(x))=-e_i.
\]
Hence \(F_1(x),\ldots,F_m(x)\) are linearly independent: if
\[
\sum_{i=1}^m a_iF_i(x)=0,
\]
then applying \(dT_x\) yields
\[
-\sum_{i=1}^m a_i e_i=0,
\]
and therefore \(a_i=0\) for all \(i\). Since
\(\dim\mathcal M=m\), the fields form a frame of \(T\mathcal M\).

Together with commutativity, the canonical form theorem for commuting
vector fields~\citep[Theorem~9.46]{lee_introduction_2012} implies that,
locally, there exist coordinates
\(s=(s_1,\ldots,s_m)\) such that
\[
F_i=\frac{\partial}{\partial s_i}.
\]
In these coordinates, time coherence becomes
\[
\frac{\partial T_j}{\partial s_i}=-\delta_{ij},
\]
so locally
\[
T_j=-s_j+a_j
\]
for some constants \(a_j\). Consequently, \(T\) is a local
diffeomorphism.

It remains to prove global injectivity. Let \(x,y\in\mathcal M\) satisfy
\[
T(x)=T(y)=t.
\]
Define
\[
\Psi_t
=
\Phi_m^{t_m}\circ\cdots\circ\Phi_1^{t_1}.
\]
By the common-reference condition,
\[
\Psi_t(x)=C=\Psi_t(y).
\]
For fixed time, each flow map \(\Phi_i^{t_i}\) is a diffeomorphism onto its
image, with inverse \(\Phi_i^{-t_i}\). Hence \(\Psi_t\) is injective, and
therefore \(x=y\). Thus \(T\) is an injective local diffeomorphism and,
consequently, a diffeomorphism onto
\[
\Omega=T(\mathcal M).
\]

Finally, since \(dT_x\) is invertible and
\[
dT_x(F_i(x))=-e_i,
\]
we obtain
\[
F_i(x)
=
-(dT_x)^{-1}e_i
=
-d(T^{-1})_{T(x)}e_i
=
-J_{T^{-1}}(T(x))e_i.
\]
Thus \(T\) and \(F\) define global flow-based coordinates, with
\(C=T^{-1}(0)\).
\end{proof}

\subsection{Lemmata for proofs of Theorem~{\ref{thm:loss_known}} and Theorem~{\ref{thm:loss_unknown}}}
\label{app:lemmata_for_loss_known}


\begin{definition}[Tangent and normal decomposition]
For a smooth manifold $M \subset \mathbb{R}^n$ and a vector field 
$F \in \mathfrak{X}(\mathbb{R}^n)$, we write $T_xM$ is the tangent space of $M$ at $x$ and $N_xM = (T_xM)^\perp$ its normal space. The projection operator onto the tangent space is denoted by $P_{T_x}M$ and onto the normal space by $P_{N_x}M$.
\end{definition}

\begin{definition}
Let $M$ be a subset of $\mathbb{R}^m$.  
We denote by $\mathcal{R}^e_M$ the set of all Euclidean $m$-dimensional
hyperrectangles whose $2^m$ vertices lie in $M$.  
\end{definition}

\begin{definition}[Projected side lengths]
Let $R$ be an $m$-dimensional hyperrectangle with vertex set $\mathcal{V}(R)$ and 
edge sets $\mathcal{E}_1(R),\dots,\mathcal{E}_m(R)$ corresponding to the 
coordinate directions.  Let
\[
\mathcal{V}_\phi(R) := \{\phi(v) : v \in \mathcal{V}(R)\}
\subset \phi(M)
\]
be the projected vertices under the flow $\phi$.

For each $i \in \{1,\dots,m\}$, we define the $i$-th projected side length by
\[
\ell_i^{\phi}(R)
:= \min_{\{v,w\}\in \mathcal{E}_i(R)}
\|\phi(v)-\phi(w)\|.
\]
We order these as
\[
\ell_{(1)}^{\phi}(R)\ \le\ \ell_{(2)}^{\phi}(R)\ \le\ \cdots\ \le\ 
\ell_{(m)}^{\phi}(R),
\]
and set
\[
\mathrm{lc}^{\phi}_1(R) := \ell^{\phi}_{(1)}(R), 
\qquad 
\mathrm{lc}^{\phi}_2(R) := \ell^{\phi}_{(2)}(R).
\]
\end{definition}

\begin{lemma}[Two-direction non-collapse bound]
\label{thm:two_direction_noncollapse}
Let $M \in \mathcal{M}_m$ with $m \ge 2$, and $\phi$ a flow, induced by pair $(F, T)$.
Then there exists a constant $c > 0$, depending only on the geometry of $M$ and the admissible pairs, such that
\[
\max_{r \in \mathcal{R}^e_{M}} \mathrm{lc}^{\phi}_2(r) \;\ge\; c \;>\; 0.
\]
\end{lemma}

\begin{proof}

\textbf{Step 1 — Equal-time distances are nondecreasing.}
\emph{(A) Infinitesimal expansion in the conformal metric.}
Fix $z_0\in U$ and set $z(t):=\Phi_t^F(z_0)$.
Let $\delta(t)$ be a variational vector along $z(t)$, i.e.
\[
\dot\delta(t)=J_F(z(t))\,\delta(t).
\]
Define the conformal energy
\[
E(t)\ :=\ \tfrac12\|\delta(t)\|_{g(z(t))}^2
\ =\ \tfrac12\,e^{\sigma(z(t))}\|\delta(t)\|^2.
\]
Differentiating and using $\dot z(t)=F(z(t))$ gives
\[
\begin{aligned}
E'(t)
&=\tfrac12\,e^{\sigma(z(t))}\big(F(z(t))\cdot\nabla\sigma(z(t))\big)\|\delta(t)\|^2
\;+\;e^{\sigma(z(t))}\langle \delta(t),J_F(z(t))\delta(t)\rangle\\
&=\tfrac12\,e^{\sigma(z(t))}\,
\delta(t)^\top\Big(\big(F(z(t))\cdot\nabla\sigma(z(t))\big)I
+\big(J_F(z(t))+J_F(z(t))^\top\big)\Big)\delta(t)\ \succeq \ 0.
\end{aligned}
\]
Hence
\[
t\ \mapsto\ \|\delta(t)\|_{g(z(t))}\quad\text{is nondecreasing.}
\]

\medskip
\emph{(B) Consequence for equal-time distances.}
Let $\gamma:[0,1]\to U$ be any $C^1$ curve from $x$ to $y$ and define
$\Gamma_t(s):=\Phi_t^F(\gamma(s))$.
Then $\partial_s\Gamma_t(s)=D\Phi_t^F(\gamma(s))\,\gamma'(s)$, and by (A),
\[
\|\partial_s\Gamma_t(s)\|_{g(\Gamma_t(s))}\ \ge\ \|\gamma'(s)\|_{g(\gamma(s))}\qquad \forall s\in[0,1].
\]
Integrating over $s$ yields the length expansion
\[
\mathrm{Length}_g(\Gamma_t)\ \ge\ \mathrm{Length}_g(\gamma).
\]
Taking the infimum over all curves $\gamma$ joining $x$ to $y$ gives
\[
d_g\big(\Phi_t^F(x),\Phi_t^F(y)\big)\ \ge\ d_g(x,y),
\]
so $t\mapsto d_g(\Phi_t^F(x),\Phi_t^F(y))$ is nondecreasing.

\medskip
\textbf{Step 2 — Construct a local $(m-1)$-dimensional rectangle inside a level set.}
By the Regular Level Set Theorem~\citep[Corollary~5.14]{lee_introduction_2012}, every nonempty regular level set of a smooth map \(M\to\mathbb R\) is an embedded submanifold of codimension one, and therefore has dimension \(m-1\). In particular, for \(m>1\), such a level set cannot reduce to a single point. Hence the level set
\[
S_\tau := T^{-1}(\tau)
\]
contains at least the two distinct points $x$ and $y$ with $T(x)=T(y)=\tau$.
Thus
\[
\operatorname{diam}(S_\tau) \ge \|x-y\| > 0,
\qquad
c_T := \sup_{\tau \in T(M)} \operatorname{diam}(S_\tau) > 0.
\]

It is a geometric fact that every Euclidean ball in $\mathbb{R}^{m-1}$ contains an inscribed axis-aligned hypercube $R$ whose nodes lie in the manifold $M$.

\medskip
\textbf{Step 3 — Propagation through the admissible flow.}
For any $x,y \in R$, we have $T(x) = T(y) = \tau$.
Therefore, by the nondecrease of equal-time distances,
\[
\|\phi(x) - \phi(y)\| 
= \|\Phi^F_{\tau}(x) - \Phi^F_{\tau}(y)\|
\ge \|x - y\|.
\]

Moreover, by \citep[Lemma~2.53 - p. 37]{lee_introduction_2012}, on the compact set under consideration there exist constants \(0<c\leq C<+\infty\) such that \[ c\|v\| \leq |v|_g \leq C\|v\|. \] Hence, \[ \|\phi(x)-\phi(y)\| = \|\Phi^F_{\tau}(x)-\Phi^F_{\tau}(y)\| \geq \|x-y\|_{g} \geq \frac{c}{C}\|x-y\|_{g_0}. \]

Hence the image $\phi(R)$ is an $(m-1)$-dimensional subset of $\phi(M)$ with the same or larger side lengths.

\medskip

\medskip
\textbf{Step 4 — Conclusion.}
The projected vertex set $\mathcal{V}_\phi(R)$, augmented with an additional zero-length coordinate is inside $\phi(M)$. 
In particular, the second smallest projected
side length satisfies
\[
\mathrm{lc}_2^{\phi}(R) > 0.
\]
Therefore, there exists a constant $c > 0$, depending only on $M$ and on the local
geometry of $(F,T)$, such that
\[
\max_{r \in \mathcal{R}^e_M} \mathrm{lc}_2^{\phi}(r) \ge c > 0.
\]
This completes the proof.
\end{proof}

\begin{lemma}[Tangency Limit from Hyperrectangle Collapse]
\label{thm:tangency_from_rectangle_collapse}
Let $M \subset \mathbb{R}^n$ be a smooth $m$-dimensional embedded manifold ($m \ge 1$), and let $\phi \in \mathcal{F}$ be a flow defined by $(F, T)$. Denote by $\mathcal{R}^e_M$ the family of $m$-dimensional euclidian hyperrectangles with vertices contained in $M$. If
\[
\max_{r \in \mathcal{R}^e_{M}} \mathrm{lc}^{\phi}_1(r) = 0,
\]
then the vector field $F$ becomes asymptotically tangent to $M$ in the sense that
\[
\max_{x \in M} \| P_{N_xM} F(x) \| = 0.
\]
\end{lemma}

\begin{proof}
We argue by contraposition. Assume
\[
\max_{x \in M} \|P_{N_xM} F(x)\| > 0.
\]
Then there exists $x_0 \in M$ and $\varepsilon > 0$ such that
\[
\|P_{N_{x_0}M} F(x_0)\| \ge \varepsilon.
\]

We choose the value $\tau_0 := T(x_0)$ and select a local chart around $x_0$ of the form
\[
\Theta : (y,s) \in U \subset \mathbb{R}^{m-1} \times \mathbb{R} \;\longmapsto\; M,
\qquad T(\Theta(y,s)) = \tau_0 + s,
\]
so that the level set $T^{-1}(\tau_0)$ corresponds to $\{s = 0\}$ and the slices
\[
S_s := \{\Theta(y,s) : y \in \mathbb{R}^{m-1}\}
\]
are the nearby $(m-1)$-dimensional level sets of $T|_M$, and the vector field $F$ crosses these
slices transversally. In particular, near $x_0$ the ``time'' direction $t$ has a
uniform normal component of size at least $\varepsilon$, while the $y$-directions
are tangent to $M$ and tangent to each $S_t$. \\

In these coordinates, we focus on the central level set
\[
S_0 := \{\Theta(y,0) : y \in \mathbb{R}^{m-1}\}.
\]
By Lemma~\ref{thm:two_direction_noncollapse} applied on $S_0$ and the associated
flow $\phi$ defined by $(F,T)$, there exists an $(m-1)$-dimensional rectangle
$R \subset S_0$ and a constant $c_T > 0$ such that all projected tangential
side lengths of $R$ satisfy
\[
\ell_i^{\phi}(R) \;\ge\; c_T, \quad i = 1,\dots,m-1.
\]
Moreover, since $\|P_{N_{x_0}M} F(x_0)\| \ge \varepsilon$, the $s$-direction at
$x_0$ is strictly transverse to $S_0$. Hence we can choose two points
$x_0^{\pm} \in M$ on the normal line through $x_0$ with $s$-coordinates $\pm\sigma$
such that their images under $\phi$ satisfy
\[
\|\phi(x_0^{+}) - \phi(x_0^{-})\| \;\ge\; c_N
\]
for some $c_N > 0$ depending only on $\varepsilon$ and the local geometry of $(F,T)$.
Combining $R$ with this normal segment yields an $m$-dimensional combinatorial
hyperrectangle in $\phi(M)$ whose $m$ side lengths are all bounded below by
$\min\{c_T, c_N\}$.

Let $c := \min\{c_T,c_N\} > 0$. Then \emph{all} $m$ side lengths of $\phi(r)$
are at least $c$, so the smallest side length satisfies
\[
\mathrm{lc}_1^{\phi}(r) \;\ge\; c.
\]
In particular,
\[
\max_{r \in \mathcal{R}^e_M} \mathrm{lc}_1^{\phi}(r) \;\ge\; c \;>\; 0,
\]
which contradicts the assumption
\[
\max_{r \in \mathcal{R}^e_M} \mathrm{lc}_1^{\phi}(r) = 0.
\]

Therefore, the contraposition yields
\[
\max_{x \in M} \|P_{N_xM} F(x)\| = 0,
\]
i.e., $F$ is everywhere tangent to $M$.
\end{proof}

\subsection{Proof of Theorem \ref{thm:loss_known}}
\label{app:proof_loss_known}
\begin{proof}
\textbf{Step A (Achievability: $\mathcal L_m=0$).} 
Because $M \in \mathcal{M}_m$ admits a global chart $\psi : M \to U \subset \mathbb{R}^m$ is a diffeomorphism, its differential $d\psi_p : T_pM \to T_{\psi(p)}U$ is a linear isomorphism for every $p \in M$ \citep[Prop.~3.6 - p. 55]{lee_introduction_2012}. Pulling back the standard coordinate vector fields on $U$ therefore yields a smooth global frame on $M$, and hence the tangent bundle $TM$ is smoothly trivial \citep[Cor.~10.20 - p. 259]{lee_introduction_2012}.
For each $i \in \{1,\dots,m\}$, this pullback of the standard coordinate vector fields write
\[
E_i \coloneqq \psi^*\!\left(\frac{\partial}{\partial u^i}\right) \in \mathfrak{X}(M).
\]
    Equivalently, $E_i$ is the unique smooth vector field on $M$ that is $\psi$-related to $\partial/\partial u^i$ in the sense of pushforward and pullback under a diffeomorphism \citep[Corollary 3.22, p.~68]{lee_introduction_2012}. Each $E_i$ spans a rank-one smooth subbundle $L_i = \mathrm{span}\{E_i\} \subset TM$, i.e., a line bundle. By the Fundamental Theorem on Flows \citep[Theorem~9.12, pp.~212--213]{lee_introduction_2012} and the Frobenius Theorem \citealp[Theorem 19.12, p497]{lee_introduction_2012}, a smooth nowhere-vanishing vector field defines a family of nonintersecting integral curves whose images form a smooth one-dimensional foliation of $M$.

Moreover, compactness of $M$ implies that every smooth vector field on $M$ is complete \citep[Corollary~9.17, p.~216]{lee_introduction_2012}. Consequently, each $E_i$ generates a globally defined smooth flow $\Phi_{E_i}^t : M \to M$ for all $t \in \mathbb{R}$. Fix a reference location $u_\star \in U$ and denote $C \coloneqq \psi^{-1}(u_\star) \in M$ (and its image under the embedding into $\mathbb{R}^n$). Since each \(E_i\) is smooth and complete, it generates a globally defined smooth flow
\(\Phi_{E_i} : \mathbb{R} \times M \to M\), which depends smoothly on both time and initial conditions \citep[Thm.~9.12, p212-p213]{lee_introduction_2012}. Using the smoothness of the flow map $(t,x) \mapsto \Phi_{E_i}^t(x)$ and standard transversality arguments (see \citealp[Transversality Theorem, p 68]{GuilleminPollack}), one may choose, for each $i$, a smooth codimension-one embedded submanifold $\Sigma_i \subset M$ that is transverse to $E_i$ and intersects each leaf of the $E_i$-foliation exactly once. Such transversals can be realized as regular level sets of smooth functions. By the Regular Level Set Theorem \citep[Cor.~5.14, p.~106]{lee_introduction_2012}, the preimage of a regular value of a smooth map $f : M \to \mathbb{R}$ is a smooth embedded submanifold of codimension equal to the dimension of the codomain, here equal to one. Choosing $f$ so that $df(E_i)\neq 0$ along the level set yields a codimension-one submanifold $\Sigma_i$ transverse to $E_i$, mearning that $E_i$ is nowhere tangent to $\Sigma_i$.

The associated hitting-time function $T_i : M \to \mathbb{R}$, defined by the unique
$t$ such that $\Phi_{E_i}^t(x)\in\Sigma_i$, is smooth, and the map
\[
\phi_i(x) \coloneqq \Phi_{E_i}^{T_i(x)}(x)
\]
projects $M$ along the integral curves of $L_i$ onto $\Sigma_i$. Iterating this
construction produces a nested sequence
\[
M=\Sigma_0 \supset \Sigma_1 \supset \cdots \supset \Sigma_m=\{C\},
\qquad \dim\Sigma_k = m-k,
\]
so that one intrinsic dimension is eliminated at each step. After $m$ iterations,
the manifold collapses to a single point, and hence
\[
\mathbb{E}\,\big\| C - (\phi_m \circ \cdots \circ \phi_1)(X) \big\|_2^2 = 0,
\]
showing that the infimum defining $\mathcal{L}_m$ is achievable.

Finally, since $M$ is a smoothly embedded submanifold of $\mathbb{R}^n$, any tangent vector field on $M$ (in particular, each $E_i$) admits a smooth extension to an ambient vector field on $\mathbb{R}^n$ by the standard extension lemma for vector fields \citep[Lemma~8.6, p.~177]{lee_introduction_2012}. Consequently, the intrinsic flow constructions on $M$ are compatible with the ambient-flow parametrization used in the definition of $\mathcal{F}$.

With this geometric achievability in place, the following paragraph addresses the remaining admissibility condition $\mathcal{L}_F g \succeq 0$.

Let the manifold be equipped with a Riemannian metric $g_0$, for instance the Euclidean metric inherited from the ambient space. For any smooth function $\sigma$, the conformal change $g=e^{2\sigma}g_0$ satisfies
the standard identity
\begin{equation}\label{eq:conformal-Lie}
\mathcal L_F g \;=\; \mathcal L_F\!\big(e^{2\sigma}g_0\big)
\;=\; e^{2\sigma}\!\big(2\,F[\sigma]\;g_0 + \mathcal L_F g_0\big).
\end{equation}
see (\citealp[Corollary 9.39 - p 230]{lee_introduction_2012}).
At a fixed point $x\in M$, write $\mathcal L_F g_0(x)$ as a $g_0(x)$-self-adjoint
operator with minimal eigenvalue $\mu_{\min}(x)\in\mathbb R$.
If we can find a smooth $\sigma$ such that
\begin{equation}\label{eq:transport}
F[\sigma](x) \;=\; a(x)
\quad\text{with}\quad
2\,a(x) + \mu_{\min}(x) \;>\; 0 \quad \forall x\in M,
\end{equation}
then from \eqref{eq:conformal-Lie} we obtain, in a $g_0$-orthonormal basis,
\[
\lambda_{\min}\big(\mathcal L_F g(x)\big)
= e^{2\sigma(x)}\;\min_i\big(2\,a(x)+\mu_i(x)\big)
\ge e^{2\sigma(x)}\big(2\,a(x)+\mu_{\min}(x)\big) \;>\; 0,
\]
hence $\mathcal L_F g(x)\succ 0$ pointwise. Thus the problem reduces to proving the existence of a smooth
solution $\sigma$ of the transport equation $F[\sigma]=a$ with $a$ as in \eqref{eq:transport}. \\

The equation
\[
F[\sigma] \;=\; \langle \nabla \sigma,\, F \rangle \;=\; a
\]
is a first-order linear partial differential equation on $M$.  
Since the vector field $F$ is smooth, nowhere vanishing, it defines a smooth one-dimensional foliation of $M$ by its integral curves \citep[Theorem~9.12, pp.~212--213]{lee_introduction_2012}.  
By the method of characteristics, such an equation admits a global smooth solution $\sigma$ obtained by integrating $a$ along the flow lines of $F$.  
Indeed, along each integral curve $\gamma(t)$ of $F$, the equation reduces to the ordinary differential equation
\[
\frac{d}{dt}\sigma(\gamma(t)) = a(\gamma(t)),
\]
which always admits a unique smooth solution when $F \neq 0$ everywhere and $F$ is bounded (hence
globally Lipschitz), as guaranteed by the Picard--Lindelöf (Cauchy--Lipschitz) theorem \cite{Teschl} on existence and uniqueness of solutions to ordinary differential equations.

\medskip
\textbf{Step B (Vanishing loss implies tangency of the learned vector fields along the manifold.).}
We aim to rule out a potential degeneracy, showing that the training procedure cannot converge to a spurious minimizer: achieving arbitrarily small loss without learning tangent directions. Specifically, we aim to prove that there does not exist a constant $\eta > 0$ such that for every $\varepsilon>0$, there exist parameters $\{(F_i,T_i)\}_{i=1}^m$ and a constant $C$ for which the loss is below $\varepsilon$ while some generator has a uniformly non-negligible normal component along the manifold. Formally,
\[
\nexists\, \eta>0\ \ \text{s.t.}\ \ \forall\,\varepsilon>0,\ \ \exists\ \{(F_i,T_i)\}_{i=1}^m,\ C\ \ \text{with}\ \ 
\mathcal{L}\big(\{F_i,T_i\},C\big)\le \varepsilon
\ \ \text{and}\ \ 
\max_{1\le i\le m}\ \sup_{x\in M}\big|\langle F_i(x),n_M(x)\rangle\big|\ge \eta,
\]
where $n_M(x)$ denotes a unit normal field on $M$. Ruling this out is crucial: otherwise one could drive the objective to (numerically) zero using flows that are not close to tangent, so low loss would not certify that the learned directions recover the tangent bundle.

We argue by contradiction. Fix $\eta>0$ and assume that for every $k\in\mathbb{N}$ there exist parameters
$\{(F_i^{(k)},T_i^{(k)})\}_{i=1}^m$ and $C^{(k)}$ such that
\begin{equation}
\label{eq:contradiction_assumption}
\mathcal{L}\big(\{F_i^{(k)},T_i^{(k)}\},C^{(k)}\big)\le \tfrac{1}{k}
\qquad\text{and}\qquad
\max_{1\le i\le m}\ \sup_{x\in M}\big\|P_{N_xM}F_i^{(k)}(x)\big\|\ge \eta.
\end{equation}
We define a subsequence such that $(F_i^{(k)},T_i^{(k)},C^{(k)})\to(F_i^\star,T_i^\star,C^\star)$ for each $i$.
Let $\phi_i^{(k)}$ and $\phi_i^\star$ be the corresponding flows.
By continuity of the flow map and the uniform implication built into the loss,
\[
(\phi_m^\star\circ\cdots\circ\phi_1^\star)(x)\equiv C^\star
\qquad\text{for all }x\in M.
\]
Define the intermediate images
$M^{(0)}:=M$ and $M^{(j)}:=\phi_j^\star\circ\cdots\circ\phi_1^\star(M)$ for $j=1,\dots,m$.

\paragraph{Step B--1: Each flow can reduce tangent rank by at most one.}
Fix $j\in\{1,\dots,m\}$ and consider $\phi_j^\star$ acting on $M^{(j-1)}$.
Lemma~\ref{thm:two_direction_noncollapse} gives a constant $c>0$ (depending only on the geometry and admissible class)
such that $\max_{r\in\mathcal{R}^e_{M^{(j-1)}}}\mathrm{lc}^{\phi_j^\star}_2(r)\ge c$.
If $\operatorname{rank}(D\phi_j^\star|_{T M^{(j-1)}})$ dropped by at least $2$ on $M^{(j-1)}$,
then $\phi_j^\star$ would collapse \emph{two} independent tangential directions, forcing the second intrinsic side length
of every mapped $m$-rectangle to vanish, i.e. $\mathrm{lc}^{\phi_j^\star}_2(r)=0$ for all such $r$,
contradicting Lemma~\ref{thm:two_direction_noncollapse}.
Hence, each $\phi_j^\star$ can decrease the tangential rank by at most one.

\paragraph{Step B--2: Exact and controlled one-dimensional drop at every step.}
Since $(\phi_m^\star\circ\cdots\circ\phi_1^\star)(M)$ is a single point, the tangential rank of the full composition
drops from $m$ to $0$, i.e. by a total of $m$.
By Step~1, each individual flow $\phi_j^\star$ can reduce the tangential rank of $M^{(j-1)}$ by at most one.
Consequently, each of the $m$ steps must induce an \emph{exact} one-dimensional rank drop.

Moreover, this rank reduction cannot occur in a degenerate or vanishing manner.
Applying the Two-direction non-collapse bound (Lemma~\ref{thm:two_direction_noncollapse}) to $M^{(j-1)}$,
there exists a constant $c>0$, independent of $j$, such that the second intrinsic side length of any Euclidean
hyperrectangle in $M^{(j-1)}$ remains uniformly bounded below after applying $\phi_j^\star$.
Therefore, collapsing exactly one tangential direction at step $j$ forces the shortest intrinsic side to vanish,
while all remaining directions retain a positive extent.
In particular, the collapse at step $j$ satisfies the quantitative condition
\[
\max_{r\in\mathcal{R}^e_{M^{(j-1)}}} \mathrm{lc}^{\phi_j^\star}_1(r)=0
\qquad\text{and}\qquad
\max_{r\in\mathcal{R}^e_{M^{(j-1)}}} \mathrm{lc}^{\phi_j^\star}_2(r)\ge c,
\]
ensuring that the one-dimensional rank drop is well-separated and occurs at every step in the composition.

\paragraph{Step B--3: tangency enforced at each step.}
Let

$$
d_j:=\dim M^{(j-1)}=m-j+1.
$$

By Step B--2, the restriction of \(\phi_j^\star\) to \(M^{(j-1)}\)
has tangential rank \(d_j-1\) everywhere. Hence the image of every
\(d_j\)-dimensional hyperrectangle
\(r\in\mathcal R^e_{M^{(j-1)}}\) has vanishing shortest intrinsic side,

$$
\mathrm{lc}^{\phi_j^\star}_1(r)=0.
$$

Therefore

$$
\max_{r\in\mathcal R^e_{M^{(j-1)}}}
\mathrm{lc}^{\phi_j^\star}_1(r)=0,
$$

and Lemma~\ref{thm:tangency_from_rectangle_collapse}, applied to the
\(d_j\)-dimensional manifold \(M^{(j-1)}\), yields

$$
F_j^\star(y)\in T_yM^{(j-1)}
\qquad
\forall\,y\in M^{(j-1)}.
$$

Thus each field is tangent to the intermediate manifold on which its
associated flow acts.

It remains to show that this tangency can be propagated back to the
original manifold \(M\). We use here the pairwise commutativity of the
limiting vector fields,

$$
[F_i^\star,F_j^\star]=0,
\qquad i\neq j.
$$

A commuting vector field is invariant under the flow of the other
field~\citep[Theorem~9.42]{lee_introduction_2012}; hence, whenever the flow is
defined,

$$
D\Phi_{F_i^\star}^{t}(x)\,F_j^\star(x)
=
F_j^\star\!\left(\Phi_{F_i^\star}^{t}(x)\right).
\label{eq:commuting_pushforward}
$$

The same identity holds for \(F_i^\star\) under its own flow
\citep[Corollary~9.43]{lee_introduction_2012}.

Fix \(j\), and suppose that \(F_j^\star\) is tangent to \(M^{(i)}\)
for some \(i<j\). For \(x\in M^{(i-1)}\), set

$$
y=\phi_i^\star(x)
=
\Phi_{F_i^\star}^{T_i^\star(x)}(x)
\in M^{(i)}.
$$

Since \(F_j^\star(y)\in T_yM^{(i)}\), and Step B--2 gives constant
tangential rank for \(\phi_i^\star\), there exists
\(v\in T_xM^{(i-1)}\) such that

$$
D\phi_i^\star(x)\,v=F_j^\star(y).
$$

Differentiating the variable-time flow map gives

$$
D\phi_i^\star(x)\,v
=
D\Phi_{F_i^\star}^{T_i^\star(x)}(x)\,v
+
F_i^\star(y)\,dT_i^\star(x)[v].
$$

Let

$$
A_i(x):=
D\Phi_{F_i^\star}^{T_i^\star(x)}(x).
$$

By~\eqref{eq:commuting_pushforward},

$$
F_j^\star(y)=A_i(x)F_j^\star(x),
\qquad
F_i^\star(y)=A_i(x)F_i^\star(x).
$$

Consequently,

$$
A_i(x)F_j^\star(x)
=
A_i(x)v
+
A_i(x)F_i^\star(x)\,dT_i^\star(x)[v].
$$

Since a fixed-time flow map is a diffeomorphism, \(A_i(x)\) is
invertible, and therefore

$$
F_j^\star(x)
=
v+dT_i^\star(x)[v]\,F_i^\star(x).
$$

Both terms on the right-hand side belong to
\(T_xM^{(i-1)}\): by construction
\(v\in T_xM^{(i-1)}\), while the first part of the present step,
applied with index \(i\), gives
\(F_i^\star(x)\in T_xM^{(i-1)}\).
Hence

$$
F_j^\star(x)\in T_xM^{(i-1)}.
$$

Starting from the tangency of \(F_j^\star\) on \(M^{(j-1)}\) and
applying this argument successively for
\(i=j-1,j-2,\ldots,1\), we obtain

$$
F_j^\star(x)\in T_xM,
\qquad
\forall\,x\in M.
$$

Equivalently,

$$
\sup_{x\in M}
\left\|P_{N_xM}F_j^\star(x)\right\|=0.
$$

Since \(j\) is arbitrary, the conclusion holds for every
\(j=1,\ldots,m\).

\paragraph{Step B--4: contradiction.}
By convergence $F_i^{(k)}\to F_i^\star$ and continuity of $x\mapsto P_{N_xM}$ on the smooth embedded manifold $M$,
we obtain
\[
\sup_{x\in M}\big\|P_{N_xM}F_i^{(k)}(x)\big\|\ \longrightarrow\
\sup_{x\in M}\big\|P_{N_xM}F_i^\star(x)\big\| \;=\; 0
\qquad\text{for each }i.
\]
Hence, for $k$ large enough,
$\max_i\sup_{x\in M}\|P_{N_xM}F_i^{(k)}(x)\|<\eta$,
contradicting \eqref{eq:contradiction_assumption}.
This proves the claim.

\medskip
\textbf{Step C (Length coordinates are charts).}
Define $\ell_j(x)$ as the signed arc-length traversed along the integral curve of $F_j$
during its time horizon (measured intrinsically on $M$) at the $j$-th stage, pulled back to $M$.
On a domain where $TM=E_1\oplus\cdots\oplus E_m$, the differential of
\[
x\longmapsto (\ell_1(x),\dots,\ell_m(x))
\]
has columns the unit vectors spanning $E_1,\dots,E_m$; hence its Jacobian is invertible.
By the inverse function theorem, $(\ell_1,\dots,\ell_m)$ form local coordinates.

\medskip
This proves achievability $\mathcal L_m=0$ and the three structural statements at any optimizer.
\end{proof}

\subsection{Proof of Theorem \ref{thm:loss_unknown}}
\label{app:proof_loss_unknown}

\begin{proof}
Let $\phi_j$ be the flow defined by $(F_j,T_j)$, and set
\[
M^{(0)}:=M,\qquad M^{(j)}:=\phi_j\circ\cdots\circ\phi_1\big(M\big)\quad (j=1,\dots,k).
\]
By Theorem~\ref{thm:two_direction_noncollapse}, for each $j$ there exists a constant
$c_j>0$ (depending only on the geometry of $M^{(j-1)}$ and admissible bounds) such that
$M^{(j)}$ contains an intrinsic $m$-dimensional hyperrectangle whose second shortest
side is at least $c_j$; in particular, a single admissible map cannot collapse two
independent directions at once.

Iterating this $k$ times with $k<m$ implies that after $k$ compositions there remains
at least $(m-k)\ge 1$ intrinsic directions with nontrivial extent. More precisely,
there exists an intrinsic $(m-k)$-dimensional hyperrectangle
\[
R_k\ \subset\ M^{(k)}=\phi_k\circ\cdots\circ\phi_1(M)
\]
whose second side length is bounded below by
\[
\mathrm{lc}^{\phi_k}_2(R_k)\ \ge\ v\ :=\ \min\{c_1,\dots,c_k\}\ >\ 0.
\]
In particular, $\operatorname{diam}\!\big(M^{(k)}\big)\ \ge\ \mathrm{lc}^{\phi_k}_2(R_k)\ \ge\ v$.

Consequently, no point $C\in\mathbb{R}^n$ can be at zero distance from all of
$M^{(k)}$; the optimal choice of $C$ in a minimax sense  must incur
a radius at least $\mathrm{lc}^{\phi_k}_2(R_k)/2$, hence
\[
\inf_{C\in\mathbb{R}^n}\ \sup_{y\in M^{(k)}} \|y-C\|\ \ge\ \frac{v}{2}.
\]
In particular, the squared error cannot vanish:
\[
\inf_{\{(F_i,T_i)\}_{i=1}^k,\ C\in\mathbb{R}^n}\ 
\sup_{x\in M}\ \big\|C-(\phi_k\circ\cdots\circ\phi_1)(x)\big\|^2
\ \ge\ \frac{v^2}{4} >\ 0.
\]

If $X\sim\rho$ is supported on $M$, then $Y:=(\phi_k\circ\cdots\circ\phi_1)(X)$ is
supported on $M^{(k)}$, so $\min_{C}\mathbb{E}\|Y-C\|^2\ge 0$ cannot be $0$.
(Indeed, since $\operatorname{diam}(M^{(k)})\ge v$, at least two points of the support
are $v$ apart, forcing a positive mean squared deviation for every $C$.)
This proves that $\mathcal L_k\ge c$ for some $c>0$ depending only on $M$ and the
admissible bounds (e.g.\ one may take $c=v^2/4$ in the minimax sense above).
\end{proof}

\subsection{Proof of Theorem \ref{thm:equiv_commuting}}
\label{app:proof_equiv_commute}

\begin{proof}
Define the composed flow
\[
S(x)\;:=\;\phi^{F_m}_{T_m(x)}\circ\cdots\circ\phi^{F_1}_{T_1(x)}(x),
\]
and the autonomous vector field
\[
G(x)\;:=\;\sum_{i=1}^m T_i(x)\,F_i(x).
\]
We prove that $S(x)=\phi^G_{\log 2}(x)$ and then the loss equivalence follows immediately.

\paragraph{Step 1: A homotopy between the identity and $S$.}
For $s \in [0,1]$, define
\[
H(s,x) := \phi^{F_m}_{\,sT_m(x)} \circ \cdots \circ \phi^{F_1}_{\,sT_1(x)}(x).
\]
Then $H(0,x) = x$ and $H(1,x) = S(x)$.

Because the fields commute, $[F_i,F_j]=0$, their flows commute and preserve each
other: for all $i,j,t$,
\[
(\phi^{F_j}_t)_* F_i = F_i.
\]
Moreover, for each $i$ and any point $z$,
\[
\frac{d}{ds}\,\phi^{F_i}_{\,sT_i(x)}(z)
= T_i(x)\,F_i\big(\phi^{F_i}_{\,sT_i(x)}(z)\big).
\]
Differentiating the $s$-dependent composition defining $H(s,x)$ and using the chain
rule, the contribution of the $i$th factor is exactly
$T_i(x)\,F_i\big(H(s,x)\big)$, since flowing by the other $F_j$ does not change $F_i$
thanks to $(\phi^{F_j}_t)_*F_i = F_i$. Summing over $i=1,\dots,m$ yields
\begin{equation}\label{eq:Hs-ODE}
\frac{d}{ds}H(s,x)
= \sum_{i=1}^m T_i(x)\,F_i\big(H(s,x)\big).
\end{equation}

\paragraph{Step 2: Evolution of $T_i$ along the curve.}
By the chain rule,
\[
\frac{d}{ds}\,T_i\!\big(H(s,x)\big)
=
\big\langle \nabla T_i\!\big(H(s,x)\big),\, \frac{\partial H(s,x)}{\partial s}\big\rangle .
\]
Insert \eqref{eq:Hs-ODE},
\[
\frac{d}{ds}\,T_i\!\big(H(s,x)\big)
=
\sum_{j=1}^m T_j(x)\,
\big\langle \nabla T_i\!\big(H(s,x)\big),\,F_j\!\big(H(s,x)\big)\big\rangle.
\]
Using the unit Jacobian/Lie-derivative condition
\[
L_{F_j}T_i=\langle\nabla T_i,F_j\rangle=-\delta_{ij},
\]
we obtain the constant derivative
\[
\frac{d}{ds}\,T_i\!\big(H(s,x)\big)
=
\sum_{j=1}^m T_j(x)\,(-\delta_{ij})
=
-T_i(x).
\]
Since $T_i(H(0,x))=T_i(x)$, integration yields the affine identity
\begin{equation}\label{eq:Ti-scaling}
T_i\!\big(H(s,x)\big)
\;=\; T_i(x)-s\,T_i(x)
\;=\; (1-s)\,T_i(x).
\end{equation}

\paragraph{Step 3: Rewriting the ODE as an autonomous flow by time change.}
Using \eqref{eq:Ti-scaling}, each coefficient in \eqref{eq:Hs-ODE} satisfies
\[
T_j(x)=\frac{1}{1-s}\,T_j\!\big(H(s,x)\big).
\]
Thus
\[
\frac{d}{ds}H(s,x)
=
\frac{1}{1-s}\,
\sum_{j=1}^m T_j\!\big(H(s,x)\big)\,F_j\!\big(H(s,x)\big)
=
\frac{1}{1-s}\,G\!\big(H(s,x)\big).
\]
Introduce the new time variable
\[
t=\int_0^s\frac{dr}{1-r}=-\log(1-s),
\qquad\frac{dt}{ds}=\frac{1}{1-s}
\]
and write $K(t,x):=H(s(t),x)$. By the chain rule,
\[
\frac{d}{dt}K(t,x)
=
\frac{d}{ds}H(s,x)\,\frac{ds}{dt}
=
\left(\frac{1}{1-s}\,G(H(s,x))\right)\cdot(1-s)
=
G\!\big(K(t,x)\big).
\]
Also $K(0,x)=H(0,x)=x$. Hence
\[
K(t,x)=\phi^t_G(x)
\quad\Longrightarrow\quad
H(s,x)=\phi^{-\log(1-s)}_G(x).
\]
Evaluating at $s=1$ gives
\[
S(x)
=
H(1,x)
=
\phi^{\infty}_G(x)
\qquad
\text{where }G=\sum_{j=1}^m T_jF_j.
\]

\paragraph{Step 4: Loss equivalence.}
For every $X$,
\[
(\phi_m\circ\cdots\circ\phi_1)(X)
=
\phi^{\infty}_G(X),
\]
so
\[
\big\|C-(\phi_m\circ\cdots\circ\phi_1)(X)\big\|_2^2
=
\big\|C-\phi^{\infty}_{\sum_{i=1}^m T_iF_i}(X)\big\|_2^2.
\]
Taking expectations and minimizing over the same parameters
$\{T_i,F_i\},C$ yields the stated reformulation of $\mathcal{L}_m$.

\end{proof}

\subsection{Proof of Theorem \ref{thm:integration_free_loss}}
\label{app:proof_integration_free_loss}
\begin{proof}
By Theorem~\ref{thm:equiv_commuting}, under the commutativity assumption and the constraint
\[
JT(x)F(x)+I_m=0_m,
\]
the original objective reduces to the collapse dynamics generated by the autonomous vector field
\[
V(x) \;=\; F(x)T(x),
\]
through the flow $\phi_V^t(x)=\phi^{t}_{\sum_i T_iF_i}(x)$. The loss therefore measures whether trajectories of $V$ drive each point $x$ toward the target point $C$.

Instead of explicitly integrating this ODE until convergence, we introduce a reference curve
\[
u_s(t,x,C)=(1-t)x+tC+t(1-t)s(t,x,C),
\]
which satisfies the endpoint conditions
\[
u_s(0,x,C)=x,
\qquad
u_s(1,x,C)=C.
\]
If this curve is everywhere tangent to the vector field $V$, i.e.
\[
\partial_t u_s(t,x,C)=V(u_s(t,x,C))
=F(u_s(t,x,C))\,T(u_s(t,x,C)),
\]
then $u_s$ follows the same dynamics as the flow generated by $V$, while connecting $x$ to $C$. Hence driving the pointwise discrepancy
\[
\bigl\|
F(u_s(t,x,C))T(u_s(t,x,C))-\partial_t u_s(t,x,C)
\bigr\|_2^2
\]
to zero enforces path-wise convergence toward $C$ without requiring explicit numerical integration of the ODE up to its asymptotic limit.

Averaging this mismatch over $x\sim X$ and $t\sim\mathbb{U}_{[0,1]}$ yields the stated integration-free surrogate objective.
\end{proof}

\subsection{Proof of Proposition \ref{prop:us_parametrization}}
\label{app:proof_us_parametrization}
\begin{proof}
The boundary conditions follow directly from the definition. Indeed,
\[
u_s(0,x)
=
s^{-1}(s(x))
=
x,
\]
and, since \(C=s^{-1}(0)\),
\[
u_s(1,x)
=
s^{-1}(0)
=
C.
\]

We now prove the non-crossing property. Fix \(t<1\), and assume that
\[
u_s(t,x_1)=u_s(t,x_2).
\]
By definition,
\[
s^{-1}\big((1-t)s(x_1)\big)
=
s^{-1}\big((1-t)s(x_2)\big).
\]
Since \(s^{-1}\) is injective, we obtain
\[
(1-t)s(x_1)=(1-t)s(x_2).
\]
Because \(t<1\), we have \(1-t>0\), hence
\[
s(x_1)=s(x_2).
\]
Since \(s\) is injective, this implies
\[
x_1=x_2.
\]
Therefore, for every \(t<1\), the map \(x\mapsto u_s(t,x)\) is injective.

Finally, we prove the time-coherence property. To avoid confusion between the
map \(s\) and the time parameter, write the first time parameter as
\(\alpha\in[0,1)\). Then
\[
u_s(\alpha,x)
=
s^{-1}\big((1-\alpha)s(x)\big).
\]
Thus
\[
s(u_s(\alpha,x))
=
(1-\alpha)s(x).
\]
Therefore,
\[
u_s(t,u_s(\alpha,x))
=
s^{-1}\big((1-t)s(u_s(\alpha,x))\big)
=
s^{-1}\big((1-t)(1-\alpha)s(x)\big).
\]
Since
\[
(1-t)(1-\alpha)
=
1-(t+\alpha-t\alpha),
\]
we obtain
\[
u_s(t,u_s(\alpha,x))
=
s^{-1}\big((1-(t+\alpha-t\alpha))s(x)\big)
=
u_s(t+\alpha-t\alpha,x).
\]
Renaming \(\alpha\) as \(s\) gives
\[
u_s(t,u_s(s,x))
=
u_s(t+s-ts,x),
\]
\end{proof}

\subsection{Proof of Proposition \ref{prop:TF_parametrization}}
\label{app:proof_us_TF_parametrization}
\begin{proof}
Since \(s\) is a diffeomorphism, for every \(x\in\mathbb R^n\),
\[
J_{s^{-1}}(s(x)) = Js(x)^{-1}.
\]
Therefore
\[
JT(x)
=
PJs(x).
\]
Using the definition
\[
F(x)=-J_{s^{-1}}(s(x))P^\top,
\]
we obtain
\[
JT(x)F(x)
=
PJs(x)\left(-J_{s^{-1}}(s(x))P^\top\right).
\]
Since \(Js(x)J_{s^{-1}}(s(x))=I_n\), this gives
\[
JT(x)F(x)
=
-PP^\top.
\]
Because
\[
P=
\begin{bmatrix}
I_m & 0
\end{bmatrix},
\]
we have \(PP^\top=I_m\). Hence
\[
JT(x)F(x)=-I_m.
\]
This proves the time-coherence identity.

We now prove commutativity. Let \(e_i\) denote the \(i\)-th canonical basis vector
of \(\mathbb R^m\), embedded in \(\mathbb R^n\) through \(P^\top e_i\). The \(i\)-th
vector field is
\[
F_i(x)
=
-J_{s^{-1}}(s(x))P^\top e_i.
\]
Applying \(Js(x)\) to \(F_i(x)\), we get
\[
Js(x)F_i(x)
=
-P^\top e_i.
\]
Thus, in the \(s\)-coordinates \(y=s(x)\), the vector field \(F_i\) becomes the
constant vector field
\[
\widetilde F_i(y)=-P^\top e_i.
\]
Therefore,
\[
[\widetilde F_i,\widetilde F_j](y)=0.
\]
Since Lie brackets are preserved by diffeomorphisms, we obtain
\[
Js(x)[F_i,F_j](x)
=
[\widetilde F_i,\widetilde F_j](s(x))
=
0.
\]
Because \(Js(x)\) is invertible, it follows that
\[
[F_i,F_j](x)=0.
\]
This proves commutativity.

Finally, using the definition of \(u_s\),
\[
u_s(t,x)=s^{-1}((1-t)s(x)).
\]
Hence
\[
s(u_s(t,x))=(1-t)s(x).
\]
Applying \(P\), we get
\[
T(u_s(t,x))
=
P s(u_s(t,x))
=
P((1-t)s(x)).
\]
Therefore
\[
T(u_s(t,x))
=
(1-t)Ps(x)
=
(1-t)T(x).
\]
This proves the compatibility with \(u_s\).
\end{proof}

\subsection{Appendix - Infinitesimal non-contracting flow-matching distortion}

\begin{proposition}[Infinitesimal non-contracting flow-matching distortion]
\label{prop:loss}
Let
\[
z_t = u_s(t,x),
\qquad
\frac{d}{d\tau}\Phi_\tau^x(z)
=
\sum_i T_i(x)F_i(\Phi_\tau^x(z)),
\qquad
\Phi_0^x(z)=z,
\]
and let \(\Phi_{i,\tau}^x\) be the flow of \(T_i(x)F_i(z)\). Define the target infinitesimal map
\[
\Psi_\tau^{x,t}(z)
=
z+\tau v_{\rm non}(x,t),
\]

Define the infinitesimal position mismatch by
\[
\mathcal P(\Phi^x,\Psi^{x,t};z_t)
:=
\lim_{\tau\downarrow0}
\frac{
\|\Phi_\tau^x(z_t)-\Psi_\tau^{x,t}(z_t)\|^2
}{\tau^2}.
\]

For each \(i\), let \((e_{i,k})_k\) be an orthonormal eigenbasis of
\[
S_i(z_t)
=
JF_i(z_t)+JF_i(z_t)^\top,
\qquad
\lambda_{i,k}(z_t)
=
\lambda_k(S_i(z_t)).
\]
Define the field-wise infinitesimal excess contraction by
\[
\mathcal C_-(\Phi^x,\Psi^{x,t};z_t)
:=
\lim_{\tau\downarrow0}
\frac1{\tau^2}
\sum_i
\|T_i(x)F_i(z_t)\|^2
\sum_k
\left[
\|D\Psi_\tau^{x,t}(z_t)e_{i,k}\|
-
\|D\Phi_{i,\tau}^x(z_t)e_{i,k}\|
\right]_+^2 .
\]

The loss is given by the sum of infinitesimal position mismatch and the infinitesimal non-contracting distortion:
\begin{align*}
    \mathcal{L}(x,t) &= \mathcal P(\Phi^x,\Psi^{x,t};z_t) + \mathcal C_-(\Phi^x,\Psi^{x,t};z_t) \\
    &= \left\| \sum_iT_i(x)F_i(z_t)-v_{\rm non}(x,t) \right\|^2 + \frac14 \sum_i \|T_i(x)F_i(z_t)\|^2
\sum_k
\left(
(- T_i(x) \lambda_{i,k}(z_t))
\right)_{+}^2.
\end{align*}
\end{proposition}

\textbf{Proof of Proposition \ref{prop:loss}}
\begin{proof}
Assume that the vector fields \(F_i\) are \(C^1\) in a neighbourhood of
\(z_t\). First, since \(\Phi_\tau^x\) is the flow of
\[
Y^x(z)=\sum_i T_i(x)F_i(z),
\]
we have the first-order expansion
\[
\Phi_\tau^x(z_t)
=
z_t
+
\tau \sum_i T_i(x)F_i(z_t)
+
o(\tau).
\]
Moreover, by definition of the target infinitesimal map,
\[
\Psi_\tau^{x,t}(z_t)
=
z_t+\tau v_{\rm non}(x,t).
\]
Therefore
\[
\Phi_\tau^x(z_t)-\Psi_\tau^{x,t}(z_t)
=
\tau
\left(
\sum_i T_i(x)F_i(z_t)-v_{\rm non}(x,t)
\right)
+
o(\tau).
\]
Hence
\[
\mathcal P(\Phi^x,\Psi^{x,t};z_t)
=
\left\|
\sum_i T_i(x)F_i(z_t)-v_{\rm non}(x,t)
\right\|^2 .
\]

We now compute the infinitesimal excess contraction term. Since
\[
\Psi_\tau^{x,t}(z)=z+\tau v_{\rm non}(x,t),
\]
its differential is
\[
D\Psi_\tau^{x,t}(z_t)=I.
\]
Thus, for any unit vector \(e_{i,k}\),
\[
\|D\Psi_\tau^{x,t}(z_t)e_{i,k}\|=1.
\]

Let
\[
Y_i^x(z)=T_i(x)F_i(z).
\]
The differential of the flow \(\Phi_{i,\tau}^x\) satisfies the variational
equation
\[
\frac{d}{d\tau}D\Phi_{i,\tau}^x(z_t)
=
JY_i^x(\Phi_{i,\tau}^x(z_t))
D\Phi_{i,\tau}^x(z_t),
\qquad
D\Phi_{i,0}^x(z_t)=I.
\]
Consequently,
\[
D\Phi_{i,\tau}^x(z_t)e_{i,k}
=
e_{i,k}
+
\tau JY_i^x(z_t)e_{i,k}
+
o(\tau).
\]
Since \(JY_i^x(z_t)=T_i(x)JF_i(z_t)\), we get
\[
D\Phi_{i,\tau}^x(z_t)e_{i,k}
=
e_{i,k}
+
\tau T_i(x)JF_i(z_t)e_{i,k}
+
o(\tau).
\]

Because \(\|e_{i,k}\|=1\), we have
\[
\|e_{i,k}+\tau a+o(\tau)\|
=
1+\tau \langle e_{i,k},a\rangle+o(\tau).
\]
Applying this with \(a=T_i(x)JF_i(z_t)e_{i,k}\), we obtain
\[
\|D\Phi_{i,\tau}^x(z_t)e_{i,k}\|
=
1
+
\tau T_i(x)
\left\langle
e_{i,k},
JF_i(z_t)e_{i,k}
\right\rangle
+
o(\tau).
\]
Now
\[
\left\langle e_{i,k},JF_i(z_t)e_{i,k}\right\rangle
=
\frac12
e_{i,k}^{\top}
\left(JF_i(z_t)+JF_i(z_t)^\top\right)
e_{i,k}.
\]
Since \(e_{i,k}\) is an eigenvector of
\[
S_i(z_t)=JF_i(z_t)+JF_i(z_t)^\top
\]
with eigenvalue \(\lambda_{i,k}(z_t)\), this gives
\[
\left\langle e_{i,k},JF_i(z_t)e_{i,k}\right\rangle
=
\frac12\lambda_{i,k}(z_t).
\]
Therefore
\[
\|D\Phi_{i,\tau}^x(z_t)e_{i,k}\|
=
1
+
\frac{\tau}{2}T_i(x)\lambda_{i,k}(z_t)
+
o(\tau).
\]

It follows that
\[
\left[
\|D\Psi_\tau^{x,t}(z_t)e_{i,k}\|
-
\|D\Phi_{i,\tau}^x(z_t)e_{i,k}\|
\right]_+
=
\left[
-\frac{\tau}{2}T_i(x)\lambda_{i,k}(z_t)
+
o(\tau)
\right]_+ .
\]
Dividing by \(\tau\), squaring, and taking \(\tau\downarrow0\), we obtain
\[
\lim_{\tau\downarrow0}
\frac1{\tau^2}
\left[
\|D\Psi_\tau^{x,t}(z_t)e_{i,k}\|
-
\|D\Phi_{i,\tau}^x(z_t)e_{i,k}\|
\right]_+^2
=
\frac14
\left(
- T_i(x)\lambda_{i,k}(z_t)
\right)_+^2.
\]
Therefore
\[
\mathcal C_-(\Phi^x,\Psi^{x,t};z_t)
=
\frac14
\sum_i
\|T_i(x)F_i(z_t)\|^2
\sum_k
\left(
- T_i(x)\lambda_{i,k}(z_t)
\right)_+^2.
\]

Combining the expressions for \(\mathcal P\) and \(\mathcal C_-\) gives
\[
\mathcal{L}(x,t)
=
\mathcal P(\Phi^x,\Psi^{x,t};z_t)
+
\mathcal C_-(\Phi^x,\Psi^{x,t};z_t),
\]
with
\[
\mathcal{L}(x,t)
=
\left\|
\sum_iT_i(x)F_i(z_t)-v_{\rm non}(x,t)
\right\|^2
+
\frac14
\sum_i
\|T_i(x)F_i(z_t)\|^2
\sum_k
\left(
- T_i(x)\lambda_{i,k}(z_t)
\right)_+^2.
\]
\end{proof}

\begin{proposition}[Variational reduction of the conformal correction]
\label{prop:conformal_reduction}
Fix \(i\) and let
\[
a_i(z)
=
\langle\nabla\sigma_i(z),F_i(z)\rangle .
\]
For the conformally Euclidean metric
\[
g_i(z)=e^{\sigma_i(z)}I_n,
\]
the Lie derivative satisfies
\[
\mathcal L_{F_i}g_i
=
e^{\sigma_i}
\left(
S_i+a_iI_n
\right),
\qquad
S_i=JF_i+JF_i^\top .
\]
Hence the infinitesimal conformal contraction rates associated with
\(T_i(x)F_i\) are
\[
T_i(x)\bigl(\lambda_{i,k}+a_i\bigr).
\]

Let
\[
b_{i,k}
=
\bigl(-T_i(x)\lambda_{i,k}\bigr)_+,
\]
and consider the conformal correction regularized by
\[
\mathcal R_{\mu,i}
=
\inf_{a\in\mathbb R}
\left\{
\sum_{k=1}^n
\left(
-T_i(x)(\lambda_{i,k}+a)
\right)_+^2
+
\mu\,T_i(x)^2a^2
\right\},
\qquad \mu>0.
\]
Then
\[
\frac{\mu}{\mu+n}
\sum_{k=1}^n b_{i,k}^2
\leq
\mathcal R_{\mu,i}
\leq
\sum_{k=1}^n b_{i,k}^2.
\]
Consequently, after minimizing over the conformal correction, the resulting
regularizer is equivalent up to positive multiplicative constants to the
Euclidean non-contraction penalty.
\end{proposition}

\begin{proof}
Set
\[
c=T_i(x)a.
\]
At the minimum one may restrict to \(c\geq0\), since \(c<0\) can only
increase both the spectral and quadratic terms. Therefore,
\[
\mathcal R_{\mu,i}
=
\inf_{c\geq0}
\left\{
\sum_{k=1}^n (b_{i,k}-c)_+^2+\mu c^2
\right\}.
\]
Taking \(c=0\) immediately gives
\[
\mathcal R_{\mu,i}
\leq
\sum_k b_{i,k}^2.
\]

For the converse bound, let
\[
r_k=(b_{i,k}-c)_+.
\]
Since \(b_{i,k}\leq r_k+c\),
\[
\|b_i\|_2
\leq
\|r\|_2+\sqrt n\,c.
\]
By Cauchy--Schwarz,
\[
\left(\|r\|_2+\sqrt n\,c\right)^2
\leq
\frac{\mu+n}{\mu}
\left(
\|r\|_2^2+\mu c^2
\right).
\]
Hence, for every \(c\geq0\),
\[
\|r\|_2^2+\mu c^2
\geq
\frac{\mu}{\mu+n}\|b_i\|_2^2.
\]
Taking the infimum over \(c\) proves the result.
\end{proof}

\newpage

\section{Appendix - Theoretical example}
\label{app:theoretical_example}
As a theoretical example, consider the two-dimensional helicoid embedded in
\(\mathbb R^3\), parameterized by

$$
\varphi(r,\theta)
=
(r\cos\theta,\; r\sin\theta,\; c\theta),
\qquad c>0 .
$$

On the manifold, a natural coordinate map is

$$
h_{\mathcal M}(x)
=
\left(
\sqrt{x_1^2+x_2^2},
\frac{x_3}{c}
\right),
$$

so that \(x=\varphi(h_{\mathcal M}(x))\). A local ambient diffeomorphism can be obtained by adding a normal coordinate \(\rho\), namely

$$
h^{-1}(r,\theta,\rho)
=
\varphi(r,\theta)
+
\rho n(r,\theta),
\qquad
n(r,\theta)
=
\frac{1}{\sqrt{r^2+c^2}}
(c\sin\theta,\,-c\cos\theta,\,r).
$$

The helicoid corresponds to the slice \(\rho=0\). Therefore, for \(x=\varphi(r,\theta)\), the parametrized collapse trajectory

$$
u_s(t,x)=h^{-1}((1-t)h(x))
$$

reduces to

$$
u_s(t,x)
=
\varphi((1-t)r,(1-t)\theta).
$$

Identifying \(s=h\), Proposition~\ref{prop:TF_parametrization} gives, with
\(P=[I_2\;0]\),

$$
T(x)=Ph(x)=(r,\theta),
\qquad
F(x)=-J_{h^{-1}}(h(x))P^\top.
$$

Hence, on the helicoid,

$$
F_1(x)=-\partial_r\varphi(r,\theta)
=
(-\cos\theta,-\sin\theta,0),
$$

and

$$
F_2(x)=-\partial_\theta\varphi(r,\theta)
=
(r\sin\theta,-r\cos\theta,-c).
$$

Thus \(u_s\), \(T\), and \(F\) are obtained from the same coordinate map \(h\), and \(u_s\) continuously transports every point toward
\(C=h^{-1}(0)=\varphi(0,0)=(0,0,0)\) without crossing in the intrinsic
\((r,\theta)\)-coordinates.

\newpage

\section{Synthetic Datasets}
\label{app:datasets}

The synthetic datasets are generated as pushforwards of low-dimensional
latent variables through known smooth embeddings. More precisely, for an
embedding
\[
f:U\subseteq\mathbb R^m\rightarrow\mathbb R^n,
\qquad
x=f(u),
\]
we first sample \(u\sim p_U\) and then evaluate \(f(u)\). When \(f\) is an
embedding, the corresponding density on the manifold with respect to its
Riemannian volume measure is
\[
p_M(f(u))
=
\frac{p_U(u)}
{\sqrt{\det\!\left(J_f(u)^\top J_f(u)\right)}}.
\]
Thus, uniform sampling in the latent variables does not in general imply
uniform sampling on the manifold. The sampling distributions used for each
dataset are detailed below.

\paragraph{Sphere.}
The unit sphere \(\mathbb S^2\subset\mathbb R^3\) is generated from
\(U_1,U_2\overset{\mathrm{i.i.d.}}{\sim}\mathcal U[0,1]\) by setting
\[
\theta=2\pi U_1,
\qquad
\phi=\arccos(2U_2-1),
\]
and
\[
f_{\mathrm{sphere}}(U_1,U_2)
=
\begin{bmatrix}
\sin\phi\cos\theta\\
\sin\phi\sin\theta\\
\cos\phi
\end{bmatrix}.
\]
Since \(\theta\) is uniform on \([0,2\pi]\) and
\(\cos\phi=2U_2-1\) is uniform on \([-1,1]\), this construction samples
points uniformly with respect to surface area on \(\mathbb S^2\).

\paragraph{Swiss roll.}
The Swiss-roll surface is defined by
\[
f_{\mathrm{swiss}}(u,v)
=
\begin{bmatrix}
u\cos u\\
v\\
u\sin u
\end{bmatrix},
\qquad
u\in[1.5,12],
\quad
v\in[-1,1].
\]
Its surface-area element is
\[
\sqrt{\det(J_f^\top J_f)}
=
\sqrt{1+u^2}.
\]
To obtain uniform sampling with respect to surface area, \(v\) is sampled
uniformly on \([-1,1]\), while \(u\) is sampled with density
\[
p_U(u)
=
\frac{\sqrt{1+u^2}}
{I(12)-I(1.5)},
\]
where
\[
I(u)
=
\frac12
\left(
u\sqrt{1+u^2}+\operatorname{arsinh}(u)
\right).
\]
In practice, if \(R\sim\mathcal U[0,1]\), we set
\[
u
=
I^{-1}
\left(
I(1.5)+R\,[I(12)-I(1.5)]
\right),
\]
where \(I^{-1}\) is evaluated numerically by bisection. Consequently, the
resulting samples are uniform with respect to the area measure of the
Swiss-roll surface.

\paragraph{Hyperbolic paraboloid.}
For intrinsic dimension \(m\) and ambient dimension \(n\geq m+1\), we draw
\[
u=(u_1,\ldots,u_m)
\sim
\mathcal U([-r,r]^m),
\qquad r=2,
\]
and define
\[
z(u)
=
\sum_{j=1}^m
s_j\frac{u_j^2}{a_j^2},
\]
where
\[
s_j=
\begin{cases}
+1, & j\leq p,\\
-1, & j>p,
\end{cases}
\qquad
p=\left\lfloor\frac{m}{2}\right\rfloor
\]
by default, and \(a_j=1\) unless specified otherwise. The embedding is
\[
f_{\mathrm{hyp}}(u)
=
\begin{bmatrix}
u\\
z(u)\\
0_{n-m-1}
\end{bmatrix}
\in\mathbb R^n.
\]
The latent variables are therefore sampled uniformly in the box, with
\[
p_U(u)=\frac{1}{(2r)^m}.
\]
This does not correspond to uniform surface-area sampling in general. Before
being used for training, each ambient coordinate is standardized using the
empirical mean \(\mu\) and standard deviation \(s\) of the generated
dataset, and the complete dataset is rescaled by \(1/4\):
\[
\widetilde f_{\mathrm{hyp}}(u)
=
\frac{1}{4}
\frac{f_{\mathrm{hyp}}(u)-\mu}{s+10^{-8}},
\]
where the division is performed coordinate-wise. Ground-truth tangent
quantities are computed from this normalized embedding.

\paragraph{Sinusoidal manifold.}
Let \(m\) and \(n>m\) denote the intrinsic and ambient dimensions,
respectively. A matrix
\[
A\in\mathbb R^{(n-m)\times m}
\]
is sampled once for each generated dataset, with independent entries
\[
A_{jk}\sim\mathcal U[1,2].
\]
The latent variables are sampled independently as
\[
u\sim\mathcal U([0,1]^m),
\]
and the embedding is
\[
f_{\mathrm{sin}}(u)
=
\begin{bmatrix}
\sin(Au)\\
u
\end{bmatrix}
\in\mathbb R^n,
\]
where the sine is applied component-wise. Hence,
\[
p_U(u)=1,
\qquad u\in[0,1]^m.
\]
The corresponding Jacobian is
\[
J_{f_{\mathrm{sin}}}(u)
=
\begin{bmatrix}
\operatorname{Diag}(\cos(Au))A\\
I_m
\end{bmatrix},
\]
so that
\[
J_{f_{\mathrm{sin}}}(u)^\top J_{f_{\mathrm{sin}}}(u)
=
I_m+
A^\top
\operatorname{Diag}\!\left(\cos^2(Au)\right)A.
\]
Consequently, uniform latent sampling does not in general induce a uniform
density with respect to manifold volume.

\paragraph{Helicoidal manifold.}
For intrinsic dimension \(m\), let \(L\) denote the number of harmonics used
for each intrinsic coordinate. We sample
\[
u=(u_1,\ldots,u_m)
\sim
\mathcal U([0,1]^m).
\]
For each \(i=1,\ldots,m\) and \(\ell=1,\ldots,L\), define
\[
\omega_\ell=2\pi\ell,
\qquad
\rho_\ell=\frac{1}{\ell},
\]
and the helical coordinates
\[
h_{i,\ell}^{\sin}(u)
=
\rho_\ell\sin(\omega_\ell u_i),
\qquad
h_{i,\ell}^{\cos}(u)
=
\rho_\ell\cos(\omega_\ell u_i).
\]
The complete embedding is
\[
f_{\mathrm{hel}}(u)
=
\left[
u,\,
\left\{
h_{i,\ell}^{\sin}(u),
h_{i,\ell}^{\cos}(u)
\right\}_{i=1,\ldots,m;\,\ell=1,\ldots,L}
\right]
\in
\mathbb R^{m(2L+1)}.
\]
The decreasing amplitudes \(\rho_\ell=1/\ell\) compensate for the increasing
frequencies, since
\[
\rho_\ell\omega_\ell=2\pi.
\]
Moreover, different latent coordinates affect disjoint coordinate blocks,
and a direct calculation gives
\[
J_{f_{\mathrm{hel}}}(u)^\top
J_{f_{\mathrm{hel}}}(u)
=
\left(1+4\pi^2L\right)I_m.
\]
The volume element is therefore constant over the latent domain. Since
\(u\) is sampled uniformly on \([0,1]^m\), the resulting helicoidal dataset
is also sampled uniformly with respect to its induced manifold volume.

\section{Appendix - Experiment setup}
\label{app:exp_setup}
\paragraph{Sample number.}For all datasets, we uniformly sample 10{,}000 points from the underlying manifold for training and evaluate the models on an additional 2{,}000 test points.

\subsection{Linear manifold sanity check}
\label{app::linear_exp}
All models are trained thanks to the Algorithm~\ref{alg:loss} minimizing the loss defined in Eq.~\ref{thm:loss_unknown}, augmented with the lie derivative metric regularization $\sum_k(-\min_i \lambda(J F_k + JF_k^T)_{+})$. Trajectories are obtained by numerically integrating the learned dynamics using the Dormand--Prince method (Dopri8), with relative and absolute tolerances set to $\texttt{rtol}=10^{-8}$ and $\texttt{atol}=10^{-8}$, respectively.

The vector fields are parameterized by a multilayer perceptron (MLP) with four hidden layers of width 32 and $\tanh$ activation functions. The output layer is followed by a spherical projection with radius $10^{-2}$ in order to control the minimum magnitude of the vector field. The time scalar functions are separate MLP with the same hidden architecture. Positivity is enforced by applying a $\mathrm{softplus}$ activation at the output. All parameters are optimized jointly using the SOAP optimizer \cite{soap} with learning rate $10^{-4}$, momentum coefficients $(\beta_1,\beta_2)=(0.95,0.95)$, weight decay $10^{-2}$, and a preconditioning update frequency of 5, over 50 epochs. The regularization term is weighted by a coefficient of $1.0$. We employ a parameter initialization scheme specifically tailored for $\tanh$ activations, following \citet{lee2025robust}.

\subsection{Validation without architecture parametrization}
\label{app::without_arch_param}
The models, vector fields, times predictor and non-autonomous vector field, are implemented as multilayer perceptrons with four hidden layers of width 32 and the LipSwish activation function \cite{lipswish_act}. All networks are initialized using Glorot normal initialization. In addition, the learnable target vector $C$ in $\mathbb{R}^d$ is initialized from a uniform distribution on $[-0.5, 0.5]^d$. Optimization is carried out using the \texttt{splus} optimizer \cite{frans2025a} with parameters $\beta_1=0.9$, $\beta_2=0.95$, weight decay $10^{-3}$, and an inverse update applied every 100 steps. The learning rate follows a linear decay schedule from $10^{-2}$ to $10^{-4}$ over 100{,}000 steps. Training is performed with a mini-batch size of 100 on 200 epochs. All results are reported over 5 random seeds. Time coherence and commutativity regularizations are error described in Appendix~\ref{app:global_consistency_metrics}

\subsection{Models and Optimization}
\label{app:models_training}

\paragraph{Model architectures.}
For our method, the diffeomorphism \(s:\mathbb{R}^n\rightarrow\mathbb{R}^n\)
is parameterized by a coupling neural spline flow~\citep{NEURIPS2019_7ac71d43}
composed of eight rational-quadratic spline coupling layers. Each spline uses
eight bins over the interval \([-4,4]\), and its parameters are predicted by
two-layer conditioning networks with \(128\) hidden units per layer.

The autoencoder baseline uses symmetric encoder and decoder MLPs with eight
hidden layers of width \(128\) and LipSwish activations
~\citep{lipswish_act}. The encoder output dimension is set to the
prescribed latent dimension, and the decoder maps this representation back
to the ambient space. All linear layers are initialized with Glorot-normal
initialization.

\paragraph{Optimization.}
All models are optimized with SPlus~\citep{frans2025a}, using
\(\beta_1=0.9\), \(\beta_2=0.95\), no weight decay, and an inverse update
every \(100\) optimization steps. Unless stated otherwise, our method is
trained for \(1000\) epochs with batches of \(50\), using a linearly
decaying learning rate from \(10^{-2}\) to \(10^{-5}\) over
\(5\times10^4\) optimization steps. For the largest setting,
\((m,n)=(20,40)\), we use \(50\) epochs with batches of \(5\), a learning
rate decreasing from \(10^{-2}\) to \(10^{-4}\) over \(10^5\) steps, and
global gradient clipping at norm \(1\). The $\alpha$ parameter is fixed to $1.0$.

The autoencoder is trained for \(1000\) epochs with batches of \(50\), using
the same SPlus configuration and a learning rate decreasing linearly from
\(10^{-2}\) to \(10^{-4}\) over \(10^5\) optimization steps.

\paragraph{AE Time function and Vector fields} For the AE baseline, we define \(T(x)=e(x)\) and \(F(x)=J_g(e(x))\), where \(e\) denotes the encoder and \(g\) the decoder (generator).

\subsection{Tangent-space metric.}
\label{app:tangent_space_metrics}
Let \(f:U\subset\mathbb R^m\rightarrow\mathbb R^n\) denote the ground-truth manifold parametrization, with \(x=f(u)\). Its Jacobian

$$
J_f(u)=\frac{\partial f}{\partial u}(u)\in\mathbb R^{n\times m}
$$

spans the tangent space \(T_xM\). The orthogonal projector onto the normal space is therefore

$$
P_N(x)
=
I_n-
J_f(u)\bigl(J_f(u)^\top J_f(u)\bigr)^{-1}J_f(u)^\top .
$$

For a learned vector field \(F_i(x)\), we measure the normal leakage by

$$
r_i(x)
=
\left\|
P_N(x)\frac{F_i(x)}{\|F_i(x)\|}
\right\|_2 .
$$

The reported tangent error is

$$
\theta_{\mathrm{tan}}
=
\arcsin\!\left(
\mathbb E_{x,i}[r_i(x)]
\right)\frac{180}{\pi}.
$$

A value of \(0^\circ\) corresponds to perfect tangency. For a random unit vector in \(\mathbb R^3\) relative to a two-dimensional tangent plane, the reference value is \(30^\circ\).

\subsection{Global-consistency metrics}
\label{app:global_consistency_metrics}

We evaluate the two differential properties underlying global consistency directly from the learned time functions \(T=(T_1,\ldots,T_m)\) and vector fields \(F=(F_1,\ldots,F_m)\).

\paragraph{Time coherence.}
Time coherence requires each vector field \(F_j\) to modify only its associated coordinate \(T_j\), at unit rate, while leaving the other coordinates unchanged:

$$
\nabla T_i(x)^\top F_j(x)=-\delta_{ij}.
$$

Equivalently,

$$
JT(x)F(x)+I_m=0.
$$

We therefore define the pointwise time-coherence error as

$$
\mathcal E_{\mathrm{time}}(x)
=
\left\|JT(x)F(x)+I_m\right\|_F^2,
$$

and report its empirical average over the evaluation samples. A zero value means that flowing along \(F_j\) changes \(T_j\) at the prescribed rate while preserving all other coordinates.

\paragraph{Commutativity.}
Pairwise commutativity requires

$$
[F_i,F_j](x)=0,
\qquad i\neq j,
$$

where

$$
[F_i,F_j]
=
JF_j\,F_i-JF_i\,F_j
$$

denotes the Lie bracket. Rather than measuring the raw bracket alone, we weight it by the corresponding flow times and define

$$
\mathcal E_{\mathrm{comm}}(x)
=
\sum_{i<j}
\bigl(T_i(x)T_j(x)\bigr)^2
\left\|[F_i,F_j](x)\right\|_2^2.
$$

This weighting reflects the effect of non-commutativity on the finite transformations associated with a sample. Indeed, for small flow times, the discrepancy between applying \(F_i\) then \(F_j\) and applying them in the opposite order is, to leading order, proportional to

$$
T_i(x)T_j(x)[F_i,F_j](x).
$$

The proposed metric therefore measures the squared leading-order displacement induced by the failure of the two transformations to commute, rather than the magnitude of the Lie bracket independently of the amount of transformation applied.

For both metrics, lower values indicate greater global consistency, and a value of zero corresponds to exact satisfaction of the associated property. In our parametrization, both quantities vanish identically by construction.

\newpage

\section{Appendix - Influence of the non-contraction weight $\alpha$}
\label{app:alpha_ablation}

\begin{figure}[h]
\centering

\begin{tikzpicture}
\begin{axis}[
    width=0.72\linewidth,
    height=0.48\linewidth,
    xmode=log,
    ymode=log,
    log basis x=10,
    log basis y=10,
    xlabel={$\alpha$},
    ylabel={Tangent error},
    xmin=5e-8,
    xmax=1.5,
    xtick={1e-7,1e-6,1e-5,1e-4,1e-3,1e-2,1e-1,1},
    grid=both,
    minor grid style={gray!20},
    major grid style={gray!30},
    axis line style={white},
    xtick pos=left,
    ytick pos=left,
    axis background/.style={fill=BackgroundGraph},
]

\addplot[
    thick,
    mark=*,
    mark options={scale=0.7,solid},
    color=BlueCustom
] coordinates {
    (1e-7,  0.00052551)
    (1e-6,  0.00044340)
    (1e-5,  0.00046807)
    (1e-4,  0.00072964)
    (1e-3,  0.00107007)
    (1e-2,  0.00161081)
    (1e-1,  0.00227782)
    (0.5,   0.00310597)
    (0.9,   0.00354854)
    (1.0,   0.00348005)
};

\addplot[
    red,
    thick,
    domain=5e-8:1.5
]{0.28180487};

\node[
    red,
    font=\small,
    anchor=south west
] at (axis cs:1e-6,0.28180487)
{$\alpha=0$};

\end{axis}
\end{tikzpicture}

\caption{\textbf{Influence of the non-contraction weight $\alpha$ on Paraboloid Dataset (20,40).} Tangent-space error as a function of $\alpha$ on logarithmic scales. The case $\alpha=0$, which cannot be represented on the logarithmic horizontal axis, yields an error of $0.2818$.}
\end{figure}

\section{Appendix - Broader Impact Statement}
This paper presents work whose goal is to advance the field of Machine Learning. There are many potential societal consequences of our work, none which we feel must be specifically highlighted here.

\end{document}